\documentclass[10pt,journal,compsoc]{IEEEtran}
\usepackage{graphicx}
\usepackage{subfigure}
\usepackage{amsmath}
\usepackage{amsthm}
\usepackage{amssymb}
\usepackage{mathrsfs}
\usepackage{slashbox}
\usepackage{float}
\usepackage[vlined,ruled,linesnumbered]{algorithm2e}
\usepackage{xcolor}

\newtheorem{thm}{Theorem}

\newtheorem{lem}{Lemma}
\newtheorem{cor}{Corollary}
\newtheorem{definition}{Definition}

\newcount\DraftStatus  
\usepackage{color}
\definecolor{darkgreen}{rgb}{0,0.5,0}
\definecolor{purple}{rgb}{1,0,1}
\newcommand{\draftnote}[2]{\ifnum\DraftStatus=1
	\marginpar{
		\tiny\raggedright
		\hbadness=10000
        \def\baselinestretch{0.8}
        \textcolor{#1}{\textsf{\hspace{0pt}#2}}}
     \fi}
\newcommand{\tianqing}[1]{\draftnote{darkgreen}{[TQ: #1]}}
\newcommand{\dayong}[1]{\draftnote{blue}{[DY: #1]}}

\begin{document}

\title{Differentially Private Multi-Agent Planning for Logistic-like Problems}

\author{Dayong Ye, Tianqing Zhu*, Sheng Shen, Wanlei Zhou and Philip S. Yu
\thanks{*Tianqing Zhu is the corresponding author. D. Ye, T. Zhu, S. Shen and W. Zhou are with the Centre for Cyber Security and Privacy and the School of Computer Science, University of Technology, Sydney, Australia. Philip S. Yu is with the Department of Computer Science, University of Illinois at Chicago, USA. Email: \{Dayong.Ye, Tianqing.Zhu, Sheng.shen-1, Wanlei.Zhou\}@uts.edu.au, psyu@cs.uic.edu. This work is supported by two ARC Projects (LP170100123 and DP190100981) from the Australian Research Council, Australia, and also by NSF under grants III-1526499, III-1763325, III-1909323, and SaTC-1930941, USA.}}

\maketitle

\begin{abstract}
Planning is one of the main approaches used to improve agents' working efficiency by making plans beforehand.
However, during planning, agents face the risk of having their private information leaked.
This paper proposes a novel strong privacy-preserving planning approach for logistic-like problems. 
This approach outperforms existing approaches by addressing two challenges:
1) simultaneously achieving strong privacy, completeness and efficiency,
and 2) addressing communication constraints.
These two challenges are prevalent in many real-world applications
including logistics in military environments and packet routing in networks.
To tackle these two challenges, our approach adopts the differential privacy technique,
which can both guarantee strong privacy and control communication overhead.
To the best of our knowledge, this paper is the first to apply differential privacy to the field of multi-agent planning 
as a means of preserving the privacy of agents for logistic-like problems.
We theoretically prove the strong privacy and completeness of our approach
and empirically demonstrate its efficiency.
We also theoretically analyze the communication overhead of our approach
and illustrate how differential privacy can be used to control it.
\end{abstract}



\section{Introduction}\label{sec:introduction}
Multi-agent planning is one of the fundamental research problems in multi-agent systems \cite{desJardins99,Ye17}.
Multi-agent planning research aims to improve agents' working efficiency by making plans in advance.
Research into collaborative multi-agent planning largely focuses on jointly automated planning \cite{Torreno17}.
During jointly automated planning, agents have to share information.
However, this kind of information sharing often results in the leaking of agents' private information.
Accordingly, to protect agents' privacy, privacy preservation is introduced into the collaborative multi-agent planning process~\cite{Shani18,Shekhar20}.
The main problem associated with privacy preservation in collaborative multi-agent planning is that of how to make plans for agents
while also preserving the privacy of each agent.

Privacy can be roughly classified into four levels:
weak privacy, strong privacy, object cardinality privacy, and agent privacy \cite{Torreno17}.
Strong privacy means that an agent, regardless of its reasoning power, cannot deduce the private information of other agents based on the information available to it.
Developing a planning method with strong privacy in distributed and communication-constrained environments is challenging for the following two reasons. 
First, it is difficult to achieve strong privacy, completeness and efficiency simultaneously \cite{Tozicka17}. 
Second, in communication-constrained environments, each agent is allowed to communicate only a limited number of times.

These two challenges are widespread in many real-world applications.
A typical application is military logistics.
In military logistics, it is vital that each military unit should strongly protect its private and sensitive facts.
Also, plans for military units must be complete and efficient to avoid any delay.
In addition, communication between units has to be constrained,
since the more communication takes place, the more likely it will be that sensitive information is leaked.

Most existing planning approaches are either weak privacy-preserving or overlook the issue of privacy preservation entirely \cite{Torreno17}.
Very few approaches are strong privacy-preserving \cite{Brafman15}.
These strong privacy-preserving planning approaches, however,
may not achieve strong privacy, completeness and efficiency simultaneously, as summarized in \cite{Tozicka17}.
Moreover, these approaches also may not work efficiently in distributed and communication-constrained environments,
as they implicitly assume that an agent can communicate directly with all other agents,
and overlook the analysis of communication overhead.

Accordingly, in this paper, we develop a novel strong privacy-preserving planning approach for distributed and communication-constrained environments. 
Our approach focuses primarily on logistic-like problems, 
which are typically used as running examples in multi-agent planning.
To achieve strong privacy, completeness and efficiency simultaneously,
we adopt the differential privacy technique.
Differential privacy is a promising privacy model,
which has been mathematically proven that
when this model is in use, an individual record being stored in or removed from a dataset makes little difference to the analytical output of the dataset \cite{Dwork06,Zhu17}.
To the best of our knowledge,
we are the first to apply differential privacy to the privacy-preserving planning problem.
Using a differential privacy mechanism to obfuscate an agent's private information can strongly preserve the agent's privacy
while also having minimal impact on the usability of the agent's private information.

Furthermore, we also address the communication-constrained environment issue by adopting the concept of a `privacy budget'.
In differential privacy, a privacy budget is applied to control privacy levels.
In our proposed approach, the privacy budget can naturally be used to control communication overhead,
with the result that only a limited number of messages are permitted during a planning phase.

In summary, the contributions of this paper are two-fold:
\begin{enumerate}
  \item Improving upon existing strong privacy-preserving planning approaches,
  our approach can achieve strong privacy, completeness and efficiency simultaneously in logistic-like problems using the differential privacy technique.
  \item Our approach is more applicable to distributed and communication-constrained logistic-like problems than existing approaches.
\end{enumerate}

The remainder of this paper is organized as follows.
In the next section, a detailed review of related work is presented. 
Then, a motivating example is given in Section \ref{sec:motivation example}.
Preliminaries are presented in Section \ref{sec:preliminaries}.
After that, the novel planning approach and the theoretical analysis are presented in Sections \ref{sec:method} and \ref{sec:theoretical analysis}, respectively.
The application of our approach to other domains is illustrated in Section \ref{sec:application}.
Next, the experimental results are provided in Section \ref{sec:experiment}.
Finally, Section \ref{sec:conclusion} concludes this paper.

\section{Related Work}\label{sec:related work}


\subsection{Weak privacy-preserving approaches}

Torreno et al. \cite{Torreno14} develop a framework known as FMAP (forward multi-agent planning).
In FMAP, agents maintain a common open list with unexplored refinement plans.
Agents then jointly select an unexplored refinement plan.
Each agent then expands the plan using a forward-chaining procedure.
Agents exchange these plans and use a distributed heuristic approach to evaluate them.
Later, based on the FMAP framework, Torreno et al. \cite{Torreno15} develop a set of global heuristic functions:
DTG (domain transition graphs) heuristic and landmarks heuristic, in order to improve the efficiency of the FMAP framework.

Stolba and Komenda \cite{Stolba17} present a multi-agent distributed and local asynchronous (MADLA) planner.
This planner adopts a distributed state-space forward-chaining multi-heuristic search.
The multi-heuristic search takes the advantages of both local and distributed heuristic searches by combining them together.
As a result, the combination of the two heuristics outperforms the two heuristics separately.

Maliah et al. \cite{Maliah17} propose a greedy privacy-preserving planner (GPPP).
In GPPP, agents collaboratively generate an abstract global plan
based on two privacy-preserving heuristics: landmark-based heuristic and privacy-preserving pattern database heuristic.
Each agent generates a local plan by extending the global plan.

\subsection{Strong privacy-preserving approaches}
Brafman \cite{Brafman15} is the first to theoretically prove strong privacy in multi-agent planning. 
He proposes an approach referred to as Secure-MAFS (secure multi-agent forward search).
Secure-MAFS extends the MAFS approach \cite{Nissim14}
by reducing the amount of information exchanged between agents.
In Secure-MAFS, agents protect their privacy by opting not to communicate a given two states to others
if these two states differ only in their private elements.
This is because other agents could possibly deduce private information through the non-private or public part of the states.

Tozicka et al. \cite{Tozicka17} investigate the limits of strong privacy-preserving planning.
They formulate three aspects of strong privacy-preserving planning: privacy, completeness, and efficiency.
They theoretically find that these three aspects are difficult to achieve at the same time for a wide class of planning algorithms.
Also, they develop a strong privacy-preserving planner
that embodies a family of planning algorithms.
The planner is based on private set intersection, 
which has been proven to be computationally secure.

Stolba et al. \cite{Stolba16a,Stolba16b,Stolba18} refine privacy metrics by quantifying the amount of privacy loss.
In this case, their analysis of privacy loss is conducted by assessing information leakage \cite{Braun09,Smith09}.
The amount of information leakage is measured as the difference between initial uncertainty and remaining uncertainty.
They also develop a general approach to compute the privacy loss of search-based multi-agent planners.
This computation is based on search tree reconstruction and classification of leaked information pertaining to the applicability of actions.

\tianqing{if we did not touch object cardinality and agent privacy preservatin, we can use a subsection titled other privacy preserving approachs}
\dayong{done}

\subsection{Other privacy-preserving approaches}


Some other existing works seem to be related to ours,
such as differentially private networks \cite{Fioretto19} and privacy-preserving distributed constraint optimization \cite{Yeoh12}.
However, the research aims of these works differ from ours. 

The research of differentially private networks mainly aims at hiding specific information contained in a network,
which may be disclosed by answering queries regarding that network.
By contrast, multi-agent privacy-preserving planning aims at collaboratively making plans
without revealing the private facts of each participating agent.
In \cite{Kasi13}, Kasiviswanathan et al. develop a set of node-differentially private algorithms to engage in the private analysis of network data.
The key concept here is to obfuscate the input graph onto the set of graphs with maximum degree below a certain threshold.
Blocki et al. \cite{Blocki13} improve accuracy in differentially private data analysis
by introducing the notion of restricted sensitivity in order to reduce noise.
Restricted sensitivity represents the sensitivity of a query only over a specific subset of all possible networks.
Proserpio et al. \cite{Proserpio14} propose a platform for differentially private data analysis: wPINQ (weighted Privacy Integrated Query).
wPINQ treats edges as a weighted dataset
on which it performs $\epsilon$-differentially private computations,
such as manipulation of records and their weights.
Thus, the presence or absence of individual edges can be masked.
Fioretto et al. \cite{Fioretto19} design a privacy-preserving obfuscation mechanism for critical infrastructure networks.
Their mechanism consists of three phases: 1) obfuscating the locations of nodes using the exponential mechanism,
2) obfuscating the values of nodes using the Laplace mechanism,
and 3) redistributing the noise introduced in the previous two phases using a bi-level optimization problem.
These works assume the existence of adversaries
while in multi-agent planning, agents are typically assumed to be honest but curious.

Research into privacy-preserving distributed constraint optimization aims at
securely coordinating the value assignment for the variables under a set of constraints in order to optimize a global objective function \cite{Fioretto18b}.
By contrast, multi-agent privacy-preserving planning aims at securely making plans that enable individual agents to achieve their goals.
Grinshpoun and Tassa \cite{Grin14} devise a novel distributed constraint optimization problem (DCOP) algorithm that preserves constraint privacy.
In their problem, a group of agents needs to compare the sum of private inputs possessed by those agents against an upper bound held by another agent.
During this comparison, none of these agents learns information on either the sum or the private inputs of other agents.
Their algorithm accomplishes this through the use of a secure summation protocol and a secure comparison protocol.
Tassa et al. \cite{Tassa19} propose a DCOP algorithm that is immune to collusion
and offers constraint, topology and decision privacy.
To achieve this goal, they adopt a secure multi-party computation protocol \cite{Ben17}
which is capable of securely comparing the cost of the current full assignment and the upper bound
and guaranteeing the security of collusion of up to half of the total agents.
From an examination of the two above-mentioned works, it can be seen that
the privacy-preserving DCOP mainly focuses on securely comparing the values of variables against an upper bound,
while multi-agent privacy-preserving planning mainly focuses on the secure computation of each individual agent.

\section{A Motivating Example}\label{sec:motivation example}
\begin{figure}[ht]
\centering
	\includegraphics[scale=0.28]{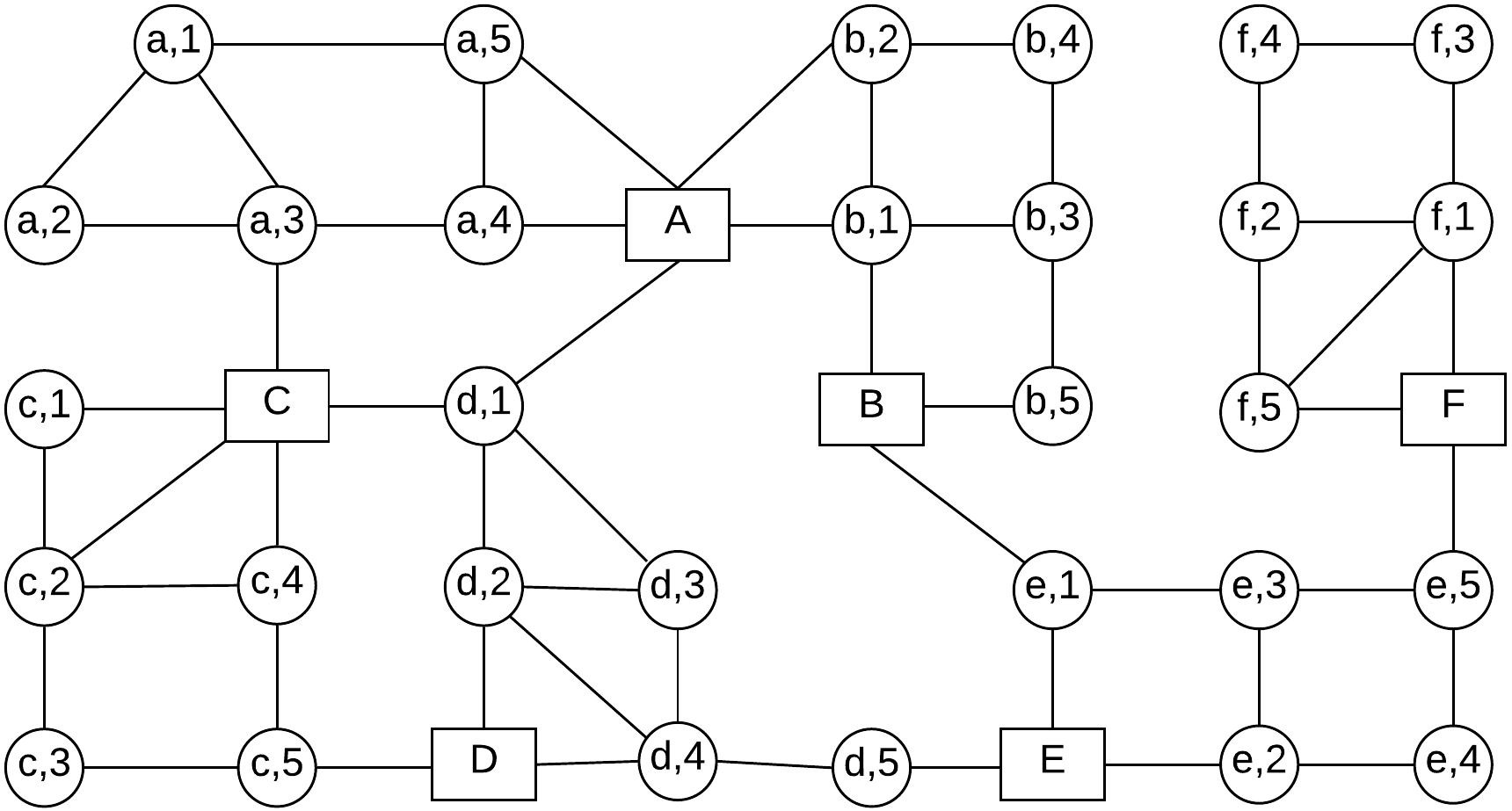}
	\caption{An example of a logistic map}
	\label{fig:example}
\end{figure}

Fig. \ref{fig:example} presents a military logistic map.
In this map, a circle denotes a military base while a rectangle denotes a logistic center.
The lines connecting the bases and logistic centers are routes.
Each route has a length, which is not indicated on the map in the interests of clarity.
Each letter in a circle indicates a military unit's name,
while each number in a circle is the index of a base in the military unit's local area.
For example, `$(a,3)$' denotes the third base in military unit $a$'s local area.
Six military units are included on this map: $a$, $b$, $c$, $d$, $e$, and $f$.
Each unit exclusively operates in a local area of the map.

Information about a local area is private to the corresponding military unit.
\tianqing{what is specific area? a,or a,3, or 3? }
This information includes 1) the number of military bases in this local area, 
2) the number of routes in this local area, 
3) the length of these routes in this local area, 
and 4) the positions of packages in this local area.
However, information regarding whether a given package is or is not located in a particular logistic center is public.
For example, in Fig. \ref{fig:areaA}, we extract military unit $a$'s local area from Fig. \ref{fig:example}.
In Fig. \ref{fig:areaA}, there are five bases: $(a,1)$, $(a,2)$, $(a,3)$, $(a,4)$ and $(a,5)$.
The number of these bases and routes is private to military unit $a$.
Moreover, the length of these routes is also private to unit $a$.
As noted above, the information that a package is located in logistic center $A$ is public and known to all military units.
\tianqing{you have to use example of depict each items, for example, it seems that you did not define distance in this area, nor topology}
\dayong{An example has been added to explain these things.}
\begin{figure}[ht]
\centering
	\includegraphics[scale=0.18]{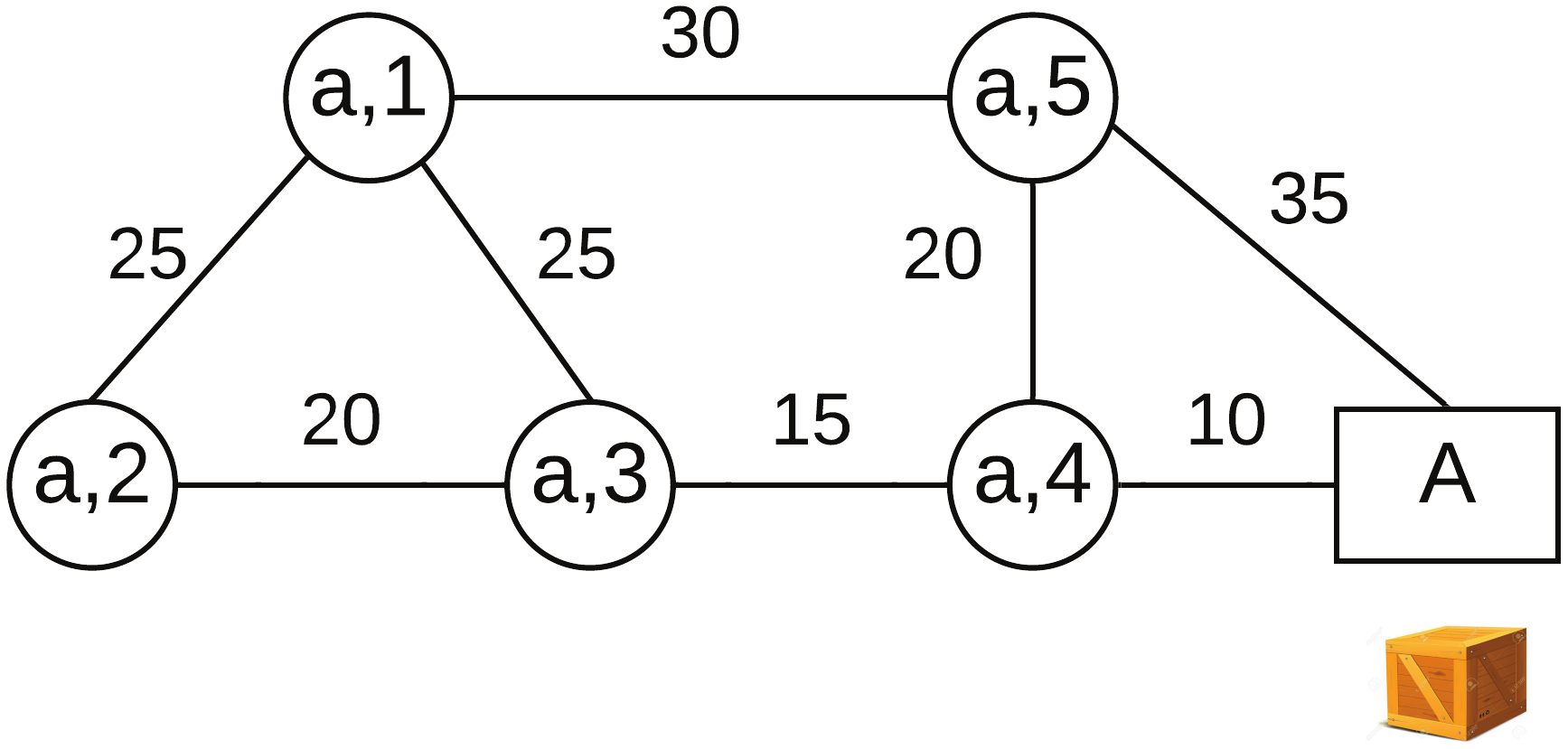}
	\caption{Unit $a$'s local area}
	\label{fig:areaA}
\end{figure}

The problem in this example is as follows:
how should a plan be made for a military unit to transport a package from one base to another,
while strongly preserving each military unit's privacy?
For example, unit $a$ wants to transport a package from $(a,2)$ to $(f,4)$, but $(f,4)$ is located in military unit $f$'s local area.
Thus, multiple units must collaborate to make a plan to deliver the package, 
while each unit's privacy is required to be strongly preserved during this process.
This problem therefore includes the above-mentioned two challenges.
First, planning for military units is highly expected to achieve strong privacy, completeness and efficiency simultaneously,
especially when military units are involved in a war.
Second, the communication of each military unit may be constrained,
as increased level of communication may result in a higher chance of private information being leaked \cite{Niu18}.

As the above two challenges have not been adequately addressed by existing approaches, 
these approaches may not be suitable for this environment.
Accordingly, in this paper, a novel strong privacy-preserving planning approach is proposed
that takes these two challenges into account.

\section{Preliminaries}\label{sec:preliminaries}
\subsection{The planning model}\label{sub:planning}
We propose a multi-agent planning model, Graph-STRIPS,
which is based on a widely used privacy-aware planning model, MA-STRIPS \cite{Brafman08}.
Graph-STRIPS is defined by a 12-tuple:
$\langle\mathcal{AG},\mathcal{V},\{\mathcal{V}_i\}_{i=1}^m,\mathcal{V}_{Pub},\mathcal{E},\{\mathcal{E}_i\}_{i=1}^m,\mathcal{P},\{\mathcal{P}_i\}_{i=1}^m,$ 
$\mathcal{A},\{\mathcal{A}_i\}_{i=1}^m,\mathcal{I},\mathcal{G}\rangle$:

\begin{itemize}
  \item $\mathcal{AG}$ is a set of agents in the environment;
  \item $\mathcal{V}$ is a set of nodes (e.g., physical entities) in the environment;
  \item $\mathcal{V}_i$ is the set of nodes private to agent $i$;
	\item $\mathcal{V}_{Pub}$ is the set of public nodes in the environment, $\mathcal{V}_{Pub}=\mathcal{V}-\cup^{|\mathcal{AG}|}_{l=1}\mathcal{V}_l$;
  \item $\mathcal{E}$ is a set of edges (e.g., the relationships between physical entities) in the environment;
  \item $\mathcal{E}_i$ is the set of edges private to agent $i$;
  \item $\mathcal{P}$ is a set of possible facts about the environment;
  \item $\mathcal{P}_i$ is the set of private facts of agent $i$;
  \item $\mathcal{A}$ is a set of possible actions of all the agents;
  \item $\mathcal{A}_i$ is the set of private actions of agent $i$;
  \item $m$ is the number of agents in the environment;
  \item $\mathcal{I}$ is the initial state of the environment;
  \item $\mathcal{G}$ is the goal state.
\end{itemize}

For example, in Fig. \ref{fig:example}, each military unit is modelled as an agent.
In this case, we have the following:
\begin{itemize}
  \item $\mathcal{AG}=\{a,b,c,d,e,f\}$ and $m=6$;
  \item $\mathcal{V}$ is the set of military bases and logistic centers;
  \item $\mathcal{V}_i$ denotes the set of bases in the local area of agent $i$;
        for example, in agent $a$'s local area, $\mathcal{V}_a=\{(a,1),(a,2),(a,3),(a,4),(a,5)\}$;
	\item $\mathcal{V}_{Pub}$ denotes the set of logistic centers;
  \item $\mathcal{E}$ is the set of routes connecting bases and centers;
  \item $\mathcal{E}_i$ denotes the set of routes in the local area of agent $i$;
        for example, in agent $a$'s local area, $\mathcal{E}_a=\{(a,1)\sim(a,2),(a,2)\sim(a,3),(a,3)\sim(a,4),...\}$;
  \item $\mathcal{P}$ includes the position of bases, logistic centers and packages;
  \item $\mathcal{P}_i$ includes 1) the position of packages in the local area of agent $i$; 
        for example, if agent $a$ has a package in $(a,1)$, then $\mathcal{P}_a=\{package\_in\_(a,1)\}$;
        2) the number of bases in the local area of agent $i$; 
				3) the number of routes in the local area of agent $i$ 
				and 4) the length of these routes.
  \item $\mathcal{A}$ includes the actions of moving from a base or a logistic center to another base or logistic center;
  \item $\mathcal{A}_i$ includes the actions of moving from a base or a logistic center to another base or logistic center in the local area of agent $i$;
        for example, an action of agent $a$ can be: \emph{moving from $(a,1$) to $(a,2)$} which is abbreviated as $(a,1)\rightarrow (a,2)$,
        where the pre-condition of this action is $package\_in\_(a,1)$ and the effect of this action is $package\_in\_(a,2)$;
  \item If $a$ wants to transport a package from $(a,3)$ to $(e,2)$,
  then $\mathcal{I}=\{package\_ in\_ (a,3)\}$ and $\mathcal{G}=\{package\_ in\_ (e,2)\}$: 
  $\mathcal{V}_{\mathcal{I}}=(a,3)$ and $\mathcal{V}_{\mathcal{G}}=(e,2)$.
\end{itemize}

If agent $a$ is to transport a package from $(a,3)$ to $(e,2)$, the associated plan could be
$\Pi^{\rhd}_a=\langle\mathcal{V}_{\mathcal{I}}\rightarrow (a,4),(a,4)\rightarrow A, A\rightarrow B, B\rightarrow E, E\rightarrow\mathcal{V}_{\mathcal{G}}\rangle$.
In plan $\Pi^{\rhd}_a$, the details of how to move from $A$ to $B$, from $B$ to $E$ and from $E$ to $(e,2)$ are not included,
as these details involve other agents' private information
that is unknown to agent $a$.
In fact, as $(e,2)$ is private to agent $e$,
agent $a$ is unaware of the existence of $(e,2)$.
Agent $a$, however, knows that the destination is in agent $e$'s local area.

Specifically, each agent’s private information includes two parts: private facts and private actions. 
An agent’s private facts include four components: 1) the number of nodes in its local area, i.e., the number of military bases in the logistic example, 
2) the number of edges in its local area, i.e., the number of routes in the logistic example, 
3) the length of these edges, i.e., the length of routes in the logistic example and 
4) the positions of any items in its local area, i.e., the positions of packages in the logistic example. 
An agent’s private actions are the movements of items in its local area. 
In this private information, the positions and movements of items are not required by other agents. 
Thus, these two pieces of information will not be disclosed to other agents. 
For the other three pieces of information: the number of nodes, the number of edges and the length of edges, 
since agents have to share the three pieces of information for planning, 
we need to develop a privacy-preserving mechanism to protect them.

Formally, we have the following definition.
\begin{definition}[Agents' privacy]\label{def:privacy}
An agent $i$'s privacy is defined as a 3-tuple: $\langle \mathcal{V}_i,\mathcal{E}_i,L(\mathcal{E}_i)\rangle$, 
where $\mathcal{V}_i$ is the set of nodes in agent $i$’s local area, 
$\mathcal{E}_i$ is the set of edges and $L(\mathcal{E}_i)$ denotes the set of length of the edges.
\end{definition}

To protect the privacy of $\mathcal{V}_i$ and $\mathcal{E}_i$, 
we adopt the node-differential privacy technique 
and uses the Laplace mechanism to mask the number of both nodes and edges. 
To protect the privacy of $L(\mathcal{E}_i)$, we adopt the exponential mechanism along with a reinforcement learning algorithm. 


\subsection{Privacy-preserving multi-agent planning}\label{sub:PPMAP}
The idea behind privacy-preserving multi-agent planning is based mainly on
research in the field of secure multi-party computation \cite{Zhao19},
where multiple agents jointly compute a function
while each agent possesses private input data.
The goal is to compute the function without revealing agents' private input data.

One intuitive solution would be to simply not disclose any private information to others.
However, since an agent must collaborate with other agents in order to achieve its goals,
it is infeasible to hide all private information completely.
To ensure that this private information is disclosed securely to the other agents, 
it is necessary to use privacy-preserving techniques.


\begin{definition}[Strong Privacy \cite{Torreno17}]\label{def:strong}
A multi-agent planning approach is strong privacy-preserving
if none of the agents is able to infer any private facts regarding an agent’s tasks
from the public information it obtains during planning.
\end{definition}
In this paper, we adopt differential privacy,
which is one of the most promising techniques in this field \cite{Dwork06}, to achieve strong privacy.

In addition to a privacy guarantee, a planning approach also needs soundness and completeness guarantees.
\begin{definition}[Soundness \cite{Maliah17}]
A planning approach is sound iff, for a given task,
there is at least one valid plan
followed by all participating agents to reach the goal state.
\end{definition}
\begin{definition}[Completeness \cite{Tozicka17}]
A planning approach is complete iff, for a given task,
1) the approach is sound and
2) the approach can guarantee to create a valid plan.
\end{definition}


\tianqing{did we use below privacies? if not, we do not need to include those definitions.}
\dayong{I have removed them.}

%

\subsection{Differential privacy}\label{sub:DP}
Differential privacy (DP) can guarantee that
any individual record being stored in or removed from a dataset will make little difference to the analytical output of the dataset \cite{Dwork06,Zhu18}.
DP has already been successfully applied to agent advising \cite{Ye19,Zhu19} and model publishing \cite{Zhu18b,Zhang20}.
Therefore, this property may also be suitable for application to the planning problem.



In differential privacy, two datasets $D$ and $D'$ are deemed neighboring datasets if they differ in only one record.
A query $f$ is a function that maps dataset $D$ to an abstract range $\mathbb{R}$: $f: D\rightarrow\mathbb{R}$.
The goal of differential privacy is to mask the differences in the answers to query $f$ between the neighboring datasets.
In $\epsilon$-differential privacy, parameter $\epsilon$ is defined as the privacy budget,
which controls the privacy guarantee level of mechanism $\mathcal{M}$.
A smaller $\epsilon$ represents stronger privacy.
The formal definition of $\epsilon$-differential privacy is as follows:

\begin{definition}[$\epsilon$-Differential Privacy \cite{Dwork14}]\label{Def-DP}
A mechanism $\mathcal{M}$ gives $\epsilon$-differential privacy for any input pair of neighboring datasets $D$ and $D'$,
and for any possible output set $\Omega$, if $\mathcal{M}$ satisfies:
\begin{equation}
Pr[\mathcal{M}(D) \in \Omega] \leq \exp(\epsilon) \cdot Pr[\mathcal{M}(D') \in \Omega]
\end{equation}
\end{definition}

In Definition \ref{Def-DP}, mechanism $\mathcal{M}$ is a function
that takes a dataset as input and outputs a query result.
Definition \ref{Def-DP} states that if a mechanism,
applied on two neighboring datasets, can obtain very similar results,
then this mechanism is a differential privacy mechanism.

Sensitivity is a parameter that captures the magnitude by which a single individual's data can change the function $f$ in the worst case.
\begin{definition}[Sensitivity \cite{Dwork14}]\label{Def-GS} For a query $f:D\rightarrow\mathbb{R}$, the sensitivity of $f$ is defined as
\begin{equation}
\Delta S=\max_{D,D'} ||f(D)-f(D')||_{1}
\end{equation}
\end{definition}

Two of the most widely used differential privacy mechanisms are the Laplace mechanism and the exponential mechanism.
The Laplace mechanism adds Laplace noise to the true answer. 
We use $Lap(b)$ to represent the noise sampled from the Laplace distribution with scaling $b$. 
\begin{definition}[Laplace mechanism \cite{Dwork14}]\label{Def-LA}
Given a function $f: D \rightarrow \mathbb{R}$ over a dataset $D$, Equation~\ref{eq-lap} is the Laplace mechanism
that provides the $\epsilon$-differential privacy \cite{Dwork14}.
\begin{equation}\label{eq-lap}
\widehat{f}(D)=f(D)+Lap(\frac{\Delta S}{\epsilon})
\end{equation}
\end{definition}

\begin{definition}[The Exponential Mechanism \cite{Dwork14}]\label{Def-Ex}
The exponential mechanism $\mathcal{M}_E$ selects and outputs an element $r\in\mathcal{R}$
with probability proportional to $exp(\frac{\epsilon u(D,r)}{2\Delta u})$,
where $u(D,r)$ is the utility of a pair of dataset and output,
and $\Delta u=\max\limits_{r\in\mathcal{R}} \max\limits_{D,D':||D-D'||_1\leq 1} |u(D,r)-u(D',r)|$
is the sensitivity of utility.
\end{definition}

If a graph is treated as a dataset, a given node in the graph can be interpreted as a record in the dataset.
According to Definition \ref{Def-DP}, we can have a similar definition for $\epsilon$-node-differential privacy as follows.
\begin{definition}[$\epsilon$-node-Differential Privacy \cite{Rask15}]\label{Def-DP-node}
A mechanism $\mathcal{M}$ gives $\epsilon$-node-differential privacy for any input pair of neighboring graphs $G$ and $G'$,
where $G$ and $G'$ differ by at most one node,
and for any possible output set, $\Omega$, if $\mathcal{M}$ satisfies:
\begin{equation}
Pr[\mathcal{M}(G) \in \Omega] \leq \exp(\epsilon) \cdot Pr[\mathcal{M}(G') \in \Omega]
\end{equation}
\end{definition}

Node-differential privacy guarantees similar output distributions on any pair of neighboring graphs 
that differ in one node and the edges adjacent to that node. 
Thus, the privacy of both nodes and edges can be preserved.


\section{The strong privacy-preserving planning approach}\label{sec:method}
In this section, we first outline our approach in a general form,
then use the aforementioned logistic example to instantiate our approach.
A generalized form of our approach is presented in Algorithm \ref{alg:general}. 
In Line 5 of Algorithm \ref{alg:general}, agent $i$ takes all the available public nodes into account to create a plan. 
These available public nodes are on the way from the initial state to the goal state 
and found by agent $i$ during its searching phase. 
However, some of these available public nodes are not needed in the final plan. 
Then, in Line 8, agent $i$ uses a reinforcement learning algorithm to find the shortest route from the initial state to the goal state, 
and selects the public nodes on the shortest route to create a plan. 
The learning is based on the information obtained in Lines 6 and 7.
\begin{algorithm}
\caption{The general form of our approach}
\label{alg:general}
/*Take agent $i\in\mathcal{AG}$ as an example;*/\\
\textbf{Input}: agent $i$'s local sets: $\mathcal{V}_i$, $\mathcal{E}_i$, $\mathcal{P}_i$, $\mathcal{A}_i$,
and all the public facts and actions;
also, the initial state $\mathcal{I}$ and the goal state $\mathcal{G}$;\\
\textbf{Output}: a complete plan $\Pi^{\rhd}_i$ from $\mathcal{I}$ to $\mathcal{G}$;\\
Agent $i$ identifies $\mathcal{V}_{\mathcal{I}}$ and $\mathcal{V}_{\mathcal{G}}$ from the initial state $\mathcal{I}$ and the goal state $\mathcal{G}$, respectively,
and initializes plan: $\Pi^{\rhd}_i=\langle \mathcal{V}_{\mathcal{I}}\rightarrow\mathcal{V}_{\mathcal{G}}\rangle$;\\
Agent $i$ searches the goal state,
and details plan $\Pi^{\rhd}_i$ by adding the available public actions into plan $\Pi^{\rhd}_i$:
$\Pi^{\rhd}_i=\langle \mathcal{V}_{\mathcal{I}}\rightarrow v_j,...,v_k\rightarrow\mathcal{V}_{\mathcal{G}}\rangle$,
where $\{v_j,...,v_k\}\subset\mathcal{V}_{Pub}$;\\
Agent $i$ queries the intermediate agents to request local private facts;\\
Each of these intermediate agents obfuscates its local private facts using the differential privacy technique;\\
Agent $i$ uses the obfuscated facts to refine the plan
by removing unnecessary public actions by means of a reinforcement learning algorithm:
$\Pi^{\rhd}_i=\langle \mathcal{V}_{\mathcal{I}}\rightarrow v_x,...,v_y\rightarrow\mathcal{V}_{\mathcal{G}}\rangle$,
where $j\leq x,y\leq k$;\\
Each action in plan $\Pi^{\rhd}_i$ is further refined by each agent creating a local plan;
for example, action $\mathcal{V}_{\mathcal{I}}\rightarrow v_x$ is refined by agent $i$ creating a local plan
as $\langle\mathcal{V}_{\mathcal{I}}\rightarrow v_{i_a},...,v_{i_b}\rightarrow v_x\rangle$,
where $\{v_{i_a},...,v_{i_b}\}\subset\mathcal{V}_i$;\\
Agent $i$ merges these local plans to form a complete plan:
$\Pi^{\rhd}_i=\langle\mathcal{V}_{\mathcal{I}}\rightarrow v_{i_a},...,v_{i_b}\rightarrow v_x,...,v_y\rightarrow\mathcal{V}_{\mathcal{G}}\rangle$;
note that the details of local plans, created by intermediate agents, are not shown in plan $\Pi^{\rhd}_i$,
since they contain non-obfuscated private facts belonging to the intermediate agents;\\
\end{algorithm}


To instantiate this general approach, we use the logistic example given in Section \ref{sec:motivation example}.
In this example, we assume that 1) all routes in the logistic map are bi-directional;
2) each individual agent controls only one local area;
and 3) there are no isolated nodes on the map.
An agent follows three steps to create a plan:
\begin{itemize}
  \item Step 1: the agent creates a high-level logistic map;
  \item Step 2: the agent asks the agents in the intermediate areas to provide route and map information;
  \item Step 3: the agent uses the received information to create a complete plan.
\end{itemize}

\subsection{Step 1: Creating a high-level map}
In Fig. \ref{fig:example}, it is supposed that agent $a$ has a package to transport from $(a,2)$ to $(f,4)$.
As agent $f$ is not agent $a$'s neighbor,
$a$ must query its neighbors, $b$, $c$, and $d$, regarding the position of $f$.
Two agents are deemed neighbors if there is at least one logistic center connecting two military bases,
such that one of these bases belongs to each of the agents.

In the case that agents, $b$, $c$ and $d$, also do not have $f$ as a neighbor,
they pass this query on to their neighbors, e.g., agent $e$.
Finally, agent $f$ is found through agent $e$.
By using the information acquired while finding agent $f$,
agent $a$ can create a high-level logistic map, as shown in Fig. \ref{fig:exampleHighLevel}.

\begin{figure}[ht]
\centering
	\includegraphics[scale=0.25]{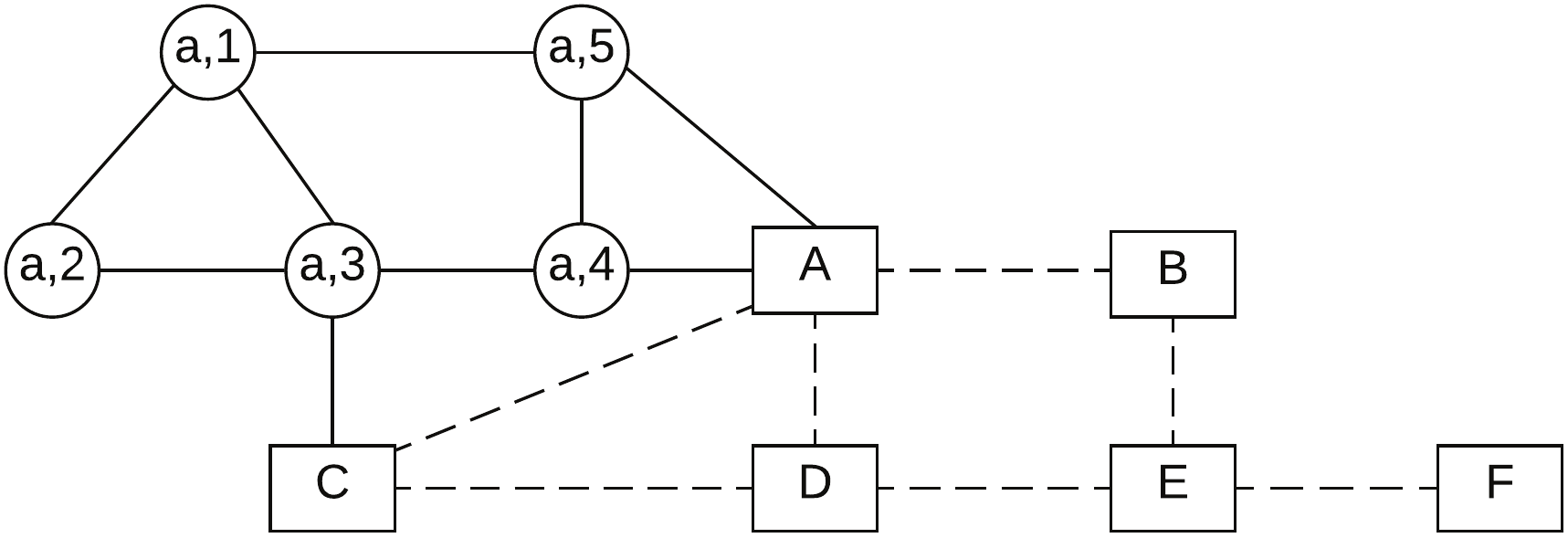}
	\caption{A high-level logistic map from agent $a$'s perspective}
	\label{fig:exampleHighLevel}
\end{figure}

%

\subsection{Step 2: Each intermediate agent provides map and route information}\label{sub:Step2}
After creating the high-level map,
agent $a$ asks the agents in the intermediate areas to provide route and map information.
In Fig. \ref{fig:exampleHighLevel}, the intermediate agents are $b$, $c$, $d$ and $e$.
To protect the topological privacy of local maps, each intermediate agent uses the Laplace mechanism
to obfuscate its local map, i.e., modify the number of bases and routes.
Moreover, to protect length privacy, 
each intermediate agent uses the exponential mechanism, along with a reinforcement learning algorithm, 
to assign probability distributions over the routes on its obfuscated local map
while removing the distance information.
Finally, each intermediate agent presents an obfuscated local map, with probability distributions over routes, to agent $a$.
An example explaining this process is presented below.


\begin{figure}[ht]
	\begin{minipage}{0.48\textwidth}
   \subfigure[\scriptsize{Before using DP}]{
    \includegraphics[scale=0.40]{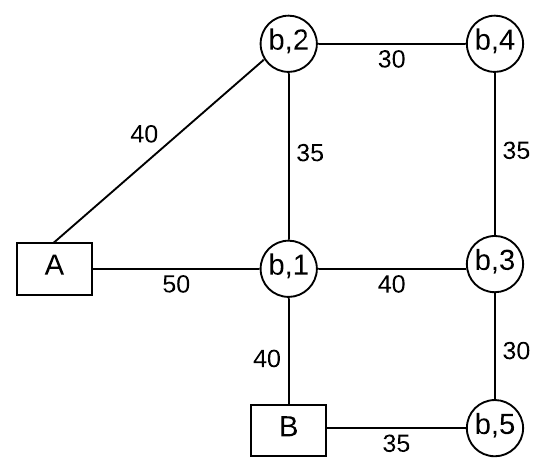}
			\label{fig:beforeDP}}
    \subfigure[\scriptsize{After using DP}]{
    \includegraphics[scale=0.40]{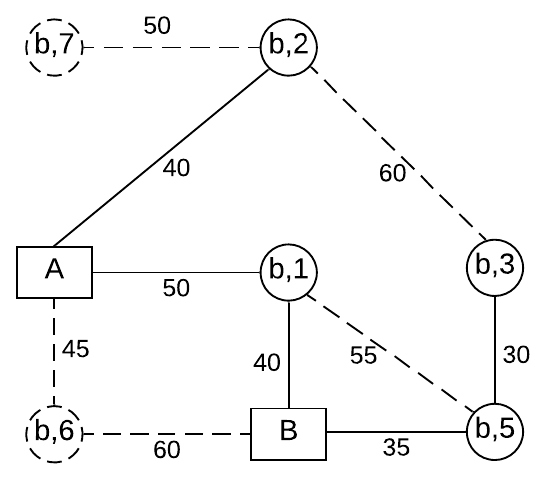}
			\label{fig:afterDP}}\\[2ex]
    \subfigure[\scriptsize{After route calculation}]{
    \includegraphics[scale=0.40]{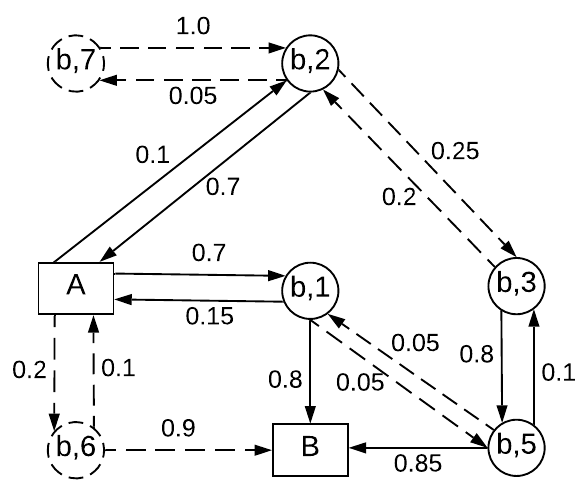}
			\label{fig:afterCalculation}}
		\subfigure[\scriptsize{After probability redistribution}]{
    \includegraphics[scale=0.40]{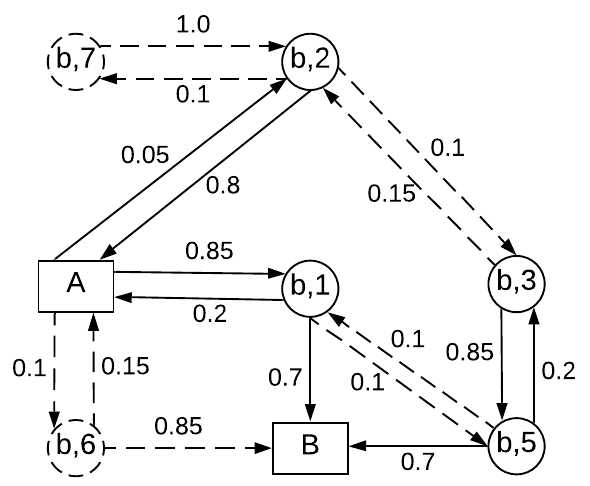}
			\label{fig:afterRedistribution}}
    \end{minipage}
	\caption{Obfuscation of agent $b$'s local map}
	\label{fig:transfer}
\end{figure}

In this example, Fig. \ref{fig:beforeDP} is agent $b$'s local map with route length.
Fig. \ref{fig:afterDP} is agent $b$'s obfuscated local map.
Referring to the obfuscated local map,
agent $b$ calculates the shortest route between logistic centers $A$ and $B$.
Then, agent $b$ marks the probability distributions over the routes, as shown in Fig. \ref{fig:afterCalculation}.
Each probability on a route indicates the probability of that route being selected. 
To guarantee the route length privacy, agent $b$ uses the exponential mechanism to redistribute these probabilities over the routes, 
as shown in Fig. \ref{fig:afterRedistribution}.
Agent $b$ then sends Fig. \ref{fig:afterRedistribution} to agent $a$.
Finally, agent $a$ receives a map
where the topology has been obfuscated
and the distance information has been replaced by probability distributions.

\subsubsection{Using the Laplace mechanism to obfuscate topologies}

The Laplace mechanism is applied to the statistical information contained in a map.
We utilize a $1K$-distribution \cite{Maha06} to obtain the statistical information.
More specifically, the $1K$-distribution is used to calculate the node degree distribution of a given graph.
To describe how the $1K$-distribution is utilized for this purpose, we employ the following example.
In Fig. \ref{fig:beforeDP}, the number of nodes with $1$ degree is $0$;
the number of nodes with $2$ degrees is $4$, (i.e., nodes $A$, $B$, $(b,4)$ and $(b,5)$);
the number of nodes with $3$ degrees is $2$, (i.e., nodes $(b,2)$ and $(b,3)$);
and the number of nodes with $4$ degrees is $1$, (i.e., node $(b,1)$).
Thus, the $1K$-distribution, i.e., the node degree distribution, of Fig. \ref{fig:beforeDP} is:
$P(1)=0$, $P(2)=4$, $P(3)=2$, and $P(4)=1$.

\begin{algorithm}
\caption{The Laplace mechanism-based obfuscation}
\label{alg:DP}
/*Take agent $b$ as an example*/\\
\textbf{Input}: agent $b$'s map (Fig. \ref{fig:beforeDP});\\
\textbf{Output}: agent $b$'s obfuscated map (Fig. \ref{fig:afterDP});\\
Use $1K$-distribution to obtain the statistical information of $b$'s map;\\
\For{$k=1$ to $d_{max}$}{
    $\tilde{P}(k)\leftarrow P(k)+\lceil Lap(\frac{\Delta S\cdot d_{max}}{\epsilon})\rceil$;\\
}
Rewire nodes to satisfy each $\tilde{P}(k)$;\\
\end{algorithm}
The Laplace mechanism-based obfuscation is outlined in Algorithm \ref{alg:DP}.
In Line 4, the statistical information of $b$'s map is obtained using the $1K$-distribution.
In Lines 5-6, the Laplace noise is added to each $P(k)$ in order to randomize the node degree distribution;
accordingly the number of nodes now becomes $\sum_{1\leq k\leq d_{max}}\tilde{P}(k)$.
Here, $d_{max}$ is the maximum node degree in a map,
and $d_{max}=4$ in the example of Fig. \ref{fig:beforeDP}.
After adding Laplace noise,
the node degree distribution could be as follows: $\tilde{P}(1)=1$, $\tilde{P}(2)=2$, $\tilde{P}(3)=5$, and $\tilde{P}(4)=0$.
Next, in Line 7, nodes are rewired to satisfy each $\tilde{P}(k)$, where $k\in\{1,...,d_{max}\}$.
The node rewiring is carried out using the graph model generator provided in \cite{Maha06}.
After node rewiring is complete, fake routes may be introduced, such as route $A\rightarrow (b,6)$ in Fig. \ref{fig:afterDP}.
The length of a fake route is randomly generated based on the average length of the existing real routes.

The reason why the Laplace mechanism is used here is that
our aim is to obfuscate the topology of each agent’s local map by modifying the degree distribution.
Since a degree distribution consists of a set of numbers,
the Laplace mechanism is more appropriate here than the exponential mechanism
which is mainly used for proportionally selecting an element from a set.
It should also be noted at this point that the Laplace mechanism may generate negative numbers.
This, however, is not a problem in this paper,
as we need both positive and negative Laplace noise to ensure that our approach satisfies $\epsilon$-differential privacy. 
Moreover, we adopt the Laplace mechanism to add noise to node degree distributions 
rather than directly adding noise to the number of nodes or edges. 
By adding noise to node degree distributions, our approach can not only guarantee the node and edge privacy of agents, 
but also guarantee the connection of an obfuscated graph. 
The connection of an obfuscated graph is a necessity for the completeness of our planning approach. 
The detailed theoretical analysis will be given in the next section.

The rationale behind Algorithm \ref{alg:DP} is as follows.
According to the definitions of differential privacy,
a map is interpreted as a dataset $D$,
while a node on a map is interpreted as a record in a dataset.
As with the concept of neighboring datasets,
two maps are deemed neighbors if they differ by only one node.
Thus, using $1K$-distribution to obtain a map's statistical information
can be thought of as querying some interesting information from a dataset, $f(D)$.
If we compare Definition \ref{Def-LA} to Line 6 in Algorithm \ref{alg:DP},
we can see that just as the Laplace mechanism can guarantee the privacy of a dataset,
it can also guarantee the privacy of a map.
More discussion about the preservation of privacy will be provided in the next section.

In Algorithm \ref{alg:DP}, $\Delta S$ represents the sensitivity of the degree distribution in a map.
The value of $\Delta S$ is determined by the maximum change in degree distribution
when a node is added into or removed from the map.
For example, in Fig. \ref{fig:beforeDP}, the degree scaling is from $1$ to $4$: $P(1),P(2),P(3),P(4)$.
According to Algorithm \ref{alg:DP}, Line 6, when a node is added into or removed from the map,
one of the four values, $P(1),P(2),P(3),P(4)$, will be incremented or decremented by $1$.
Thus, the maximum change of degree distribution is $1$, 
i.e., $\Delta S=1$ in Algorithm \ref{alg:DP}.

\subsubsection{Using reinforcement learning to compute probability distributions}

In a local area, such as the one in Fig. \ref{fig:beforeDP},
there is a set of local military bases and logistic centers,
along with a set of routes connecting these bases and centers.
As discussed in Section \ref{sec:preliminaries}, in the Graph-STRIPS model, $\mathcal{V}$ and $\mathcal{E}$ can be used to represent the topology of a map.
Accordingly, we use $\mathcal{V}$ to represent the military bases and logistic centers,
while $\mathcal{E}$ is used to denote the set of routes connecting these bases and centers.
Specifically, in Fig. \ref{fig:beforeDP}, $\mathcal{V}_b=\{(b,1),(b,2),(b,3),(b,4),(b,5)\}$,
and $\mathcal{E}_b=\{A\sim (b,1), A\sim(b,2), ...,(b,1)\sim B, (b,5)\sim B\}$.
Moreover, different bases or centers will have different routes available to them.
For example, in base $(b,1)$, there are four available routes: $\mathcal{E}_{(b,1)}=\{(b,1)\sim A, (b,1)\sim (b,2), (b,1)\sim (b,3), (b,1)\sim B\}$.
Furthermore, in center $A$, there are two available routes: $\mathcal{E}_A=\{A\sim (b,1), A\sim (b,2)\}$.


\begin{algorithm}
\caption{The reinforcement learning algorithm}
\label{alg:RL}
/*Take agent $b$ as an example*/\\
\textbf{Input}: agent $b$'s obfuscated map (Fig. \ref{fig:afterDP});\\
\textbf{Output}: agent $b$'s obfuscated map with probability distributions (Fig. \ref{fig:afterCalculation});\\
Initialize probability distributions;\\
Initialize the Q-value of each route;\\
Initialize the current position: $v\leftarrow A$;\\
\While{$v\neq B$}{
    Agent $b$ selects a route, $e$, based on the probability distribution $\boldsymbol{\pi}(v)=\langle\pi(v,e_1),...,\pi(v,e_n)\rangle$,
    where $e\in\mathcal{E}_v=\{e_1,...,e_n\}$;\\
    $r\leftarrow\mathcal{R}(v,e)$;\\
    $Q(v,e)\leftarrow (1-\alpha)Q(v,e)+\alpha[r+\gamma\underset{e_i\in\mathcal{E}_{v'}}{max_{e_i}} Q(v',e_i)]$;\\
    $\overline{r}\leftarrow\sum_{e_i\in\mathcal{E}_v}\pi(v,e_i)Q(v,e_i)$;\\
    \For{each route $e_i\in\mathcal{E}_v$}{
      $\pi(v,e_i)\leftarrow\pi(v,e_i)+\zeta(Q(v,e_i)-\overline{r})$;\\
    }
    $\boldsymbol{\pi}(v)\leftarrow Normalise(\boldsymbol{\pi}(v))$;\\
    $v\leftarrow v'$;\\
}
Agent $b$ marks the learned probability distributions over the routes;\\
\end{algorithm}
The reinforcement learning algorithm is outlined in Algorithm \ref{alg:RL}.
In Line 4, agent $b$ proportionally initializes probability distributions over actions,
where each action indicates the selection of a route.
The initialization is based on the lengths of the routes.
For example, in Fig. \ref{fig:afterDP}, the probability distribution over routes $A\sim (b,1)$ and $A\sim (b,6)$
can be initialized as $\frac{4}{9}$ and $\frac{5}{9}$, respectively.
In Line 5, agent $b$ initializes the Q-value of each route; 
here, the Q-value is an indication of how good a route is.
In this algorithm, the initial Q-value of a route is set based on the length of the route,
such that a shorter route is allocated a higher Q-value.
For example, in Fig. \ref{fig:afterDP}, the initial Q-value of route $A\sim (b,1)$ can be set to $\frac{100}{50}=2$,
while the initial Q-value of route $A\sim (b,6)$ can be set to $\frac{100}{40}=2.5$.
In Line 6, agent $b$ sets the initial position to $A$ and the destination to $B$.
This setting is based on the fact that,
as an intermediate agent, agent $b$ will help agent $a$ to transport the package from $A$ to $B$.

Regarding the loop, in Line 8, agent $b$ selects a route $e$ based on the probability distribution over the available routes in base $v$.
After taking route $e$, agent $b$ receives a reward $r$ (Line 9),
which is inversely proportional to the route length.
For example, in Fig. \ref{fig:afterDP}, $r(A\sim (b,1))$ and $r(A\sim (b,6))$ can be set to $4$ and $5$, respectively.
The reward $r$ is used to update the Q-value of route $e$ in base $v$ (Line 10).
This update is based on: 1) the current Q-value of $e$ in base $v$, $Q(v,e)$;
2) the maximum Q-value of the routes in new base $v'$, $\underset{e_i\in\mathcal{E}_{v'}}{max_{e_i}} Q(v',e_i)$;
3) the immediate reward $r$;
and 4) a learning rate $\alpha$ and a discount rate $\gamma$.
In the next step, the updated Q-value and the probability distribution are used to compute the average reward $\overline{r}$ (Line 11),
where $\mathcal{E}_v$ is the set of available routes in base $v$.
In Lines 12 and 13, the probability of selecting each route $i\in\mathcal{E}_v$ is updated.
This update is based on: 1) the current probability of each route being selected $\pi(v,e_i)$;
2) the current Q-value of each route $Q(v,e_i)$;
3) the average reward $\overline{r}$; and
4) a learning rate $\zeta$.
In Line 14, the updated probability distribution is normalized to be valid,
meaning that for each $i\in\mathcal{E}_v$, $0<\pi(v,e_i)<1$ and $\sum_{e_i\in\mathcal{E}_v}\pi(v,e_i)=1$.
In Line 15, the new base, $v'$, is set as the current base.
The above steps are iterated over until the goal state is reached.
Finally, in Line 16, agent $b$ marks each of the routes with the learned probability distributions.

\subsubsection{Using the exponential mechanism to redistribute probabilities}
After using the reinforcement learning algorithm to replace distance information with probability distributions, 
agents' local distance information can be hidden. 
Hiding distance information can reduce the risk of leaking this information 
but cannot guarantee the privacy preservation of this information. 
Therefore, we adopt the exponential mechanism to redistribute probabilities.

We use an example to explain how to use the exponential mechanism to redistribute probabilities. 
Suppose a node in a local map has two adjacent edges, $x$ and $y$, 
and the probabilities of selecting $x$ and $y$ are $0.7$ and $0.3$, respectively. 
Based on the definition of exponential mechanism, the exponential mechanism selects and outputs an element $r$ with probability proportional to $exp(\frac{\epsilon u_r}{2\Delta u})$, 
where $\epsilon$ is the privacy budget, $u_r$ is the utility of selecting $r$ and $\Delta u$ is the sensitivity of utility. 
If we set the utility of selecting a route to be the probability of selecting that route, 
then we have: $u_x=0.7$ and $u_y=0.3$, and in this setting, $\Delta u=1$. 
Then, if we set $\epsilon=2$, we have $exp(\frac{\epsilon u_x}{2\Delta u})=2.014$ and $exp(\frac{\epsilon u_y}{2\Delta u})=1.350$. 
Finally, the probabilities of selecting $x$ and $y$ become $\frac{2.014}{2.014+1.350}=0.6$ and $\frac{1.350}{2.014+1.350}=0.4$, respectively. 
The above process is performed on each node in the local map.

Another simple way to preserve the distance information privacy is to let each agent use the Dijkstra's algorithm \cite{Dijkstra59} 
to compute the shortest route length between two logistic centers in its local area 
and add a Laplace noise to that length. 
However, other agents may still get an approximate idea about the route length. 
For example, after adding a Laplace noise, the route length changes from $100$ to $105$. 
Although other agents cannot deduce the real length, 
they can still guess that the real length must be near $105$. 
In some situations, e.g., the military logistic example, an approximate length is good enough for other agents. 
By contrast, if an agent uses reinforcement learning and shares only probabilities, 
other agents cannot obtain even an approximate length. 
This idea is based on the spirit of federated learning by allowing agents to share only parameters \cite{Yang19}. 
In federated learning, to protect each client’s training data privacy, 
each client only sends the model parameters, trained based on her private data, to the server. 
The server, thus, has only clients' model parameters without any clients' private data.

\subsection{Step 3: Creating a complete plan}
After receiving obfuscated local maps from intermediate agents,
agent $a$ creates a logistic map by combining these obfuscated local maps, as shown in Fig. \ref{fig:overviewObfuscated}. 
On each obfuscated local map, although both real and fake nodes and edges are involved,
agent $a$ is unable to determine whether a given node or edge is real.
More detailed discussion on this matter will be presented in Section \ref{sec:theoretical analysis}.
\begin{figure}[ht]
\centering
	\includegraphics[scale=0.45]{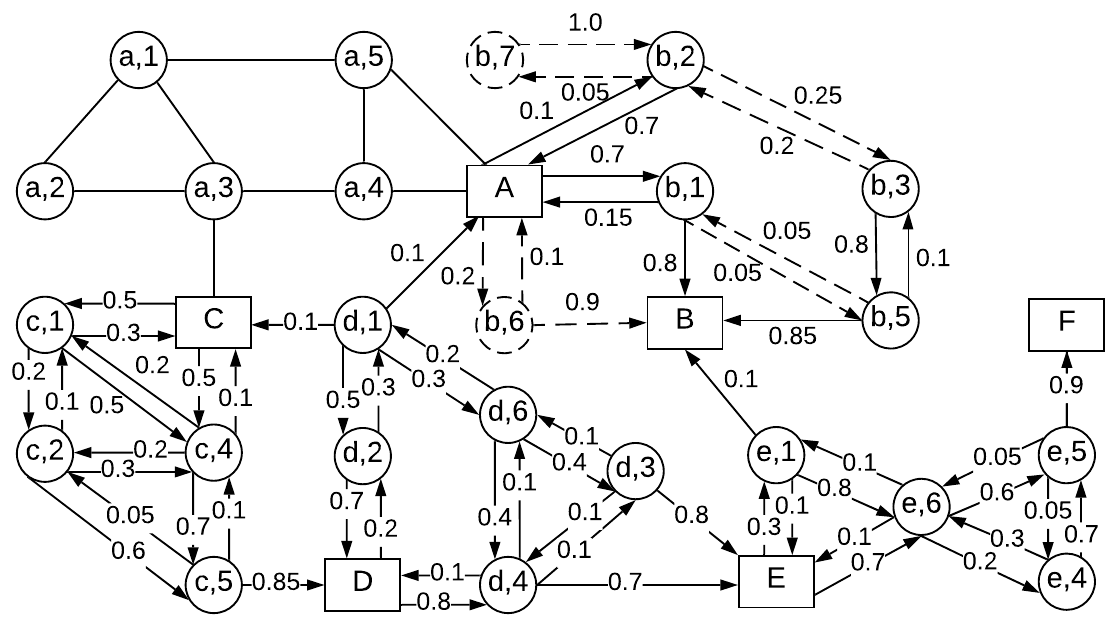}
	\caption{A logistic map created by obfuscated local maps}
	\label{fig:overviewObfuscated}
\end{figure}

Next, agent $a$ uses a reinforcement learning algorithm to calculate the length of the route between each pair of connected logistic centers, e.g., $A\rightarrow B$, $B\rightarrow E$ and so on.
The reinforcement learning algorithm is similar to Algorithm \ref{alg:RL}.
Since agent $a$ is only provided with probability distributions about the other areas,
agent $a$ must generate the distance information itself based on the probability distributions.
Agent $a$ relates the probabilities to the distance based on the average route length in agent $a$’s local area.
For example, in Fig. \ref{fig:overviewObfuscated}, the probabilities of selecting routes $A\sim (b,1)$ and $A\sim (b,6)$ are $0.7$ and $0.2$, respectively.
If the average route length in agent $a$’s local area is 45,
agent $a$ can simply set the distances from $A$ to $(b,1)$ and $A$ to $(b,6)$ to $20$ and $70$, respectively, whose average is $45$.
Here, we operate under the assumption that there are no significant differences between the average route length in each local area.

After agent $a$ calculates the length of the shortest route between each pair of connected logistic centers (as shown in Fig. \ref{fig:exampleA}), 
the shortest route from the origin to the destination can also be obtained.
It is clear at this point that this calculation is not very accurate,
as it is based on estimated length.
However, the aim of this calculation is not to find the real shortest route,
rather to select the intermediate agents which are located on the shortest route.
In Fig \ref{fig:exampleA}, the agents on the shortest route are: $b$, $e$ and $f$.

\begin{figure}[ht]
\centering
	\includegraphics[scale=0.28]{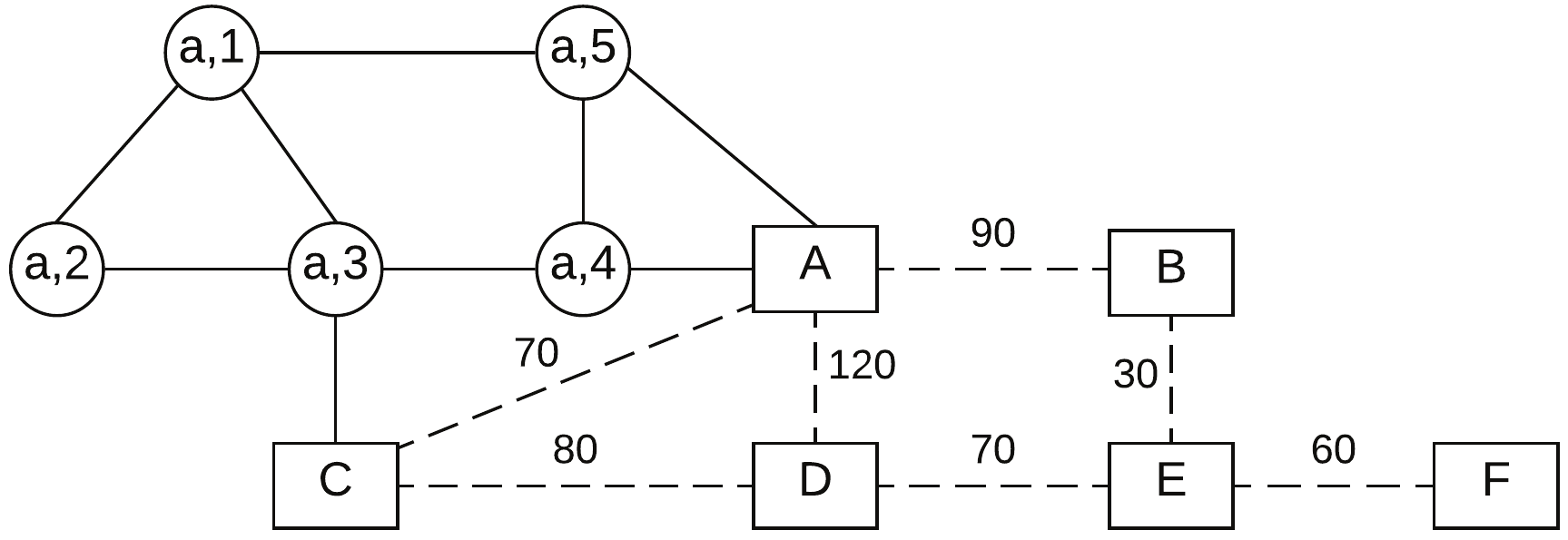}
	\caption{A high-level map featuring relative distances from agent $a$'s perspective}
	\label{fig:exampleA}
\end{figure}

The final plan, thus, can be expressed as $\Pi^{\rhd}_a=\langle\mathcal{I}\rightarrow (a,3)\rightarrow (a,4)\rightarrow A\rightarrow B\rightarrow E\rightarrow F\rightarrow\mathcal{G}\rangle$,
where $\mathcal{I}=\{package\_in\_(a,2)\}$ and $\mathcal{G}=\{package\_in\_(f,4)\}$.
In this plan, $\mathcal{I}\rightarrow (a,3)\rightarrow (a,4)\rightarrow A$ is the local plan formulated and carried out by agent $a$.
At logistic center $A$, agent $a$ gives its package to agent $b$, which makes a local plan to transport the package to logistic center $B$.
At center $B$, agent $e$ takes control of the package and devises a local plan to deliver the package to logistic center $F$.
Finally, agent $f$ picks up the package at center $F$ and makes a local plan to transfer the package to $(f,4)$.

\tianqing{at the end of this section, can we generate a new logistic map to show the 'new map' after the privacy preservation? }
\dayong{I have added a figure to show this new map.}

\subsection{A simplification of the proposed approach}
In some situations, if the distance information is not private, 
we can let logistic centers do the routing planning and consider the routing only between logistic centers. 
Each logistic center can directly communicate with the agent that is connected with the logistic center. 
As the distance information is not private, 
each logistic center is also aware of the local routing information within the agent. 
Compared with the proposed approach,
this simplified approach can 1) significantly reduce the problem complexity;
and 2) enable agents to obtain accurate distance information for further calculation.
and 3) fully hide the topology information belonging to each agent from other agents.

A typical example is daily logistic, where the distance information between two public places do not need to be hidden. In daily logistic, packages are transported from their starting points to their destinations across multiple states or provinces. Here, the distance among states/provinces is not a privacy concern and can be considered as public information. The simplified version of our approach can be applied to this example. Each state/province is assumed to have a logistic center. To transport a package, the logistic center at the starting point utilizes the accurate distance information among states/provinces to make an optimal global plan. Then, each logistic center in the global planning path conducts the local routing planning.

\section{Theoretical analysis}\label{sec:theoretical analysis}
\subsection{Soundness analysis}
\begin{thm}\label{thm:soundness}
The proposed approach is sound.
\end{thm}
\begin{proof}
We prove this theorem by considering one task, e.g., delivering one package in the logistics example.
In Step 1 of our approach, we start from the initial agent
which has a task to complete and initializes a plan,
each queried agent sets up a link to the querying agent.
Thus, all the queried agents are reachable.
If a goal agent is identified whose private facts include the goal state,
there must be at least one plan connecting the initial agent to the goal agent through some or all of the queried agents.
\end{proof}

\subsection{Completeness analysis}
\begin{lem}\label{lem:step3}
Obfuscating local maps does not affect the completeness of the proposed approach.
\end{lem}
\begin{proof}
In Step 2 of our approach, each intermediate agent obfuscates its local map
by adding and/or removing nodes and/or edges (see Algorithm \ref{alg:DP}).
During the obfuscation process, Laplace noise is added to the node degree distribution of the local map: $P(1)$, ..., $P(d_{max})$.
As $P(0)$ is not counted, isolated nodes will not be created.
Moreover, as the obfuscated map is undirected,
it can be guaranteed that the obfuscated map will be connected.
Hence, there must be at least one route between the two logistic centers on the local map.
Since this property is common to the local maps of all intermediate areas,
there must be at least one route from the initial area to the goal area via intermediate logistic centers.
Thus, the completeness is not affected.
\end{proof}

\begin{thm}\label{thm:completeness}
The proposed approach is complete.
\end{thm}
\begin{proof}
Step 1 of our approach guarantees that a goal agent can be found.
According to Theorem \ref{thm:soundness}, there must be at least one plan connecting the initial agent to the goal agent.
We now need only to prove that our approach is capable of finding at least one of these plans.

According to Lemma \ref{lem:step3}, there is at least one route from the initial area to the goal area.
One of these routes can be treated as a high-level plan,
which can be identified using Algorithm \ref{alg:RL}.
Based on the high-level plan, each intermediate agent creates a local plan (Step 3).
Given that each agent is honest\footnote{It is a common assumption in privacy-preserving multi-agent planning that agents are honest but curious about others' private information \cite{Torreno17}.}, each local plan is valid,
which ensures that the two logistic centers in the local area will be connected.
Therefore, a high-level plan and a set of local plans constitute a complete plan.
\end{proof}

\subsection{Privacy-preserving analysis}
\begin{thm}\label{thm:DP}
The proposed planning approach satisfies $\epsilon$-differential privacy.
\end{thm}
\begin{proof}
To analyze the privacy guarantee, we apply two composite properties of the privacy budget: the sequential and the parallel compositions \cite{McSherry200794}.
The sequential composition determines the privacy budget $\epsilon$ of each step
when a series of private analysis are performed sequentially on a dataset.
The parallel composition corresponds to the case in which each private step is applied to disjoint subsets of a dataset.
The ultimate privacy guarantee depends on the step which has the maximal $\epsilon$.

In the proposed approach, the Laplace mechanism and the exponential mechanism consumes the privacy budget.
In the Laplace mechanism in Algorithm \ref{alg:DP}, the Laplace noise sampled from $Lap(\frac{\Delta S\cdot d_{max}}{\epsilon})$ is added in $d_{max}$ steps.
At each step, the Laplace mechanism consumes the $\frac{\epsilon}{d_{max}}$ privacy budget;
thus for each step, Algorithm \ref{alg:DP} satisfies $\frac{\epsilon}{d_{max}}$-differential privacy.
By using the sequential composition property, we can conclude that
at a total of $d_{max}$ steps, the Laplace mechanism consumes the $d_{max}\cdot\frac{\epsilon}{d_{max}}=\epsilon$ privacy budget,
meaning that Algorithm \ref{alg:DP} satisfies $\epsilon$-differential privacy. 
By comparing Definition \ref{Def-DP} with Definition \ref{Def-DP-node}, 
since the Laplace mechanism can guarantee the data record privacy of a dataset, 
it can also guarantee the node-privacy of a graph.

The exponential mechanism is used to redistribute probabilities on each agent's local graph. 
For a given node in a local graph, suppose the node has $k$ adjacent edges. 
Then, the exponential mechanism will be used $k$ times. 
If we set privacy budget for this node to be $\frac{\epsilon}{k}$, 
based on the sequential composition property, the privacy consumption of this node is $\epsilon$. 
Thus, the probability redistribution on the adjacent edges of this node satisfies $\epsilon$-differential privacy. 
When this method is used on every node, 
based on the parallel composition property, the probability redistribution on this local graph satisfies $\epsilon$-differential privacy.

Since the Laplace mechanism and the exponential mechanism are used by each agent,
each agent is guaranteed $\epsilon$-differential privacy. 
Although an environment may contain multiple agents,
each agent maintains a local area, and these local areas are disjoint with each other.
Since each agent is guaranteed $\epsilon$-differential privacy,
according to the parallel composition property, the proposed approach satisfies $\epsilon$-differential privacy.

\end{proof}

\textbf{Remark 1}: In Algorithm \ref{alg:DP}, Laplace noise is used to randomize the node degree distribution.
This implies that both the number of nodes and the number of edges in a local map will be perturbed.
Since the topology of a map consists of nodes and edges,
perturbing the numbers of nodes and edges incurs perturbation of the topology.
Accordingly, as Algorithm \ref{alg:DP} satisfies differential privacy,
the perturbation of the topology of a map also satisfies differential privacy.

\begin{cor}\label{cor:granularity}
No agent is able to conclude anything about the existence of any subset of $\lceil\frac{\Delta S\cdot d_{max}}{\epsilon}\rceil$ nodes in another agent's map.
\end{cor}
\begin{proof}
In Algorithm \ref{alg:DP}, the Laplace noise is sampled from $Lap(\frac{\Delta S\cdot d_{max}}{\epsilon})$,
meaning that the expected amount of noise is $\frac{\Delta S\cdot d_{max}}{\epsilon}$.
As this noise is used to change the number of nodes in a map (recall Lines 5-6 in Algorithm \ref{alg:DP}),
the expected number of nodes that will be changed is $\lceil\frac{\Delta S\cdot d_{max}}{\epsilon}\rceil$.
Therefore, any subset of $\lceil\frac{\Delta S\cdot d_{max}}{\epsilon}\rceil$ nodes could be fake nodes.
According to Definition \ref{Def-DP-node} and Theorem \ref{thm:DP},
since Algorithm \ref{alg:DP} can guarantee the node-privacy of a graph,
an agent will be unable to distinguish real from fake statistical information between two neighboring graphs, e.g., the number of real nodes.
This means that an agent cannot determine	whether or not a node is fake.
Hence, the existence of any subset of $\lceil\frac{\Delta S\cdot d_{max}}{\epsilon}\rceil$ nodes in an agent's map
cannot be concluded by any other agents.
\end{proof}

\textbf{Remark 2}: From Corollary \ref{cor:granularity}, in the Laplace mechanism in Algorithm \ref{alg:DP}, 
the value of $\epsilon$ controls the granularity of privacy,
given that the values of $\Delta S$ and $d_{max}$ have been fixed.
A smaller $\epsilon$ implies a stronger privacy guarantee.
However, a smaller $\epsilon$ also introduces a larger amount of noise.
The increase of the amount of noise reduces the usability of a map.
Thus, the value of $\epsilon$ should be carefully set.

\textbf{Remark 3}: Similar to the Laplace mechanism, in the exponential mechanism, 
the value of $\epsilon$ has a huge impact on probability redistribution results. 
Given that a node has $k$ adjacent edges and the probabilities of selecting the $k$ edges are $u_1,...,u_k$, 
if we set $\epsilon=0$, the probability of selecting each edge will equally become $\frac{1}{k}$; 
if we set $\epsilon\rightarrow +\infty$, probability $u_m$ becomes $1$ and others become $0$, 
where $u_m=max\{u_1,...,u_k\}$. 
In addition to the two extreme situations, there is a median situation 
which is that the redistributed probabilities are identical to the original probabilities: $u'_1=u_1,...,u'_k=u_k$. 
Based on the computation method described in Section \ref{sub:Step2}, each probability $u'_i$, $1\leq i\leq k$, is computed as: 
\begin{equation}
u'_i=\frac{exp(\frac{\epsilon u_i}{2\Delta u})}{\sum_{1\leq j\leq k}exp(\frac{\epsilon u_j}{2\Delta u})}. 
\end{equation}
Let each $u'_i=u_i$, we have $k$ equations.
\begin{equation}\nonumber
\left\{
             \begin{array}{lr}
             \frac{exp(\frac{\epsilon u_1}{2\Delta u})}{\sum_{1\leq j\leq k}exp(\frac{\epsilon u_j}{2\Delta u})}=u_1, &  \\
             ..., & \\
             \frac{exp(\frac{\epsilon u_k}{2\Delta u})}{\sum_{1\leq j\leq k}exp(\frac{\epsilon u_j}{2\Delta u})}=u_k. &  
             \end{array}
\right.
\end{equation}
In our problem, $\Delta u=1$. By solving the $k$ equations, we have that 
\begin{equation}\nonumber
\epsilon_i=\frac{2(k\cdot ln(u_i)-\sum_{1\leq j\leq k}ln(u_j))}{k\cdot u_i-\sum_{1\leq j\leq k}u_j}, 
\end{equation}
where $1\leq i\leq k$. 
Thus, in applications, on one hand, these values of $\epsilon$ should be avoided, 
as they will make the redistributed probabilities identical to the original probabilities, 
which cannot offer any privacy preservation. 
On the other hand, the values of $\epsilon$ should be set close to these values 
to guarantee the usability of the redistributed probabilities.


\begin{thm}\label{thm:strong}
The proposed planning approach can strongly preserve agents' privacy.
\end{thm}
\begin{proof}
As defined in Section \ref{sub:planning}, an agent's private information includes
1) the number of nodes in an agent's local area, 
2) the number of edges in the local area,  
3) the length of these edges, 
4) the positions of any items in the local area and 
5) the movements of any items in the local area.
To prove this theorem,
we only need to prove that the private information possessed by an agent cannot be inferred by another agent.
First, according to Theorem \ref{thm:DP} and Corollary \ref{cor:granularity},
the proposed planning approach satisfies $\epsilon$-differential privacy
and guarantees the privacy of any subset of $\lceil\frac{\Delta S\cdot d_{max}}{\epsilon}\rceil$ nodes in an agent's local area.
By properly setting the value of $\epsilon$, the privacy of all nodes and edges in an agent's local area can be preserved.
Therefore, the privacy of the number of nodes and edges of an agent's local area will also be preserved.
Second, our approach dictates that the length information in a local area is replaced by probability distributions (recall Fig. \ref{fig:transfer}). 
Also, these probabilities are redistributed using the exponential mechanism. 
Thus, the length information is strictly hidden. 
Therefore, an agent cannot infer the real length of any individual edge in another agent's local area.
Third, since the privacy of any node or edge in an agent's local area has been preserved,
the positions and movements of items have also been preserved.
Based on the definition of strong privacy (Definition \ref{def:strong}),
the proposed approach can strongly preserve agents' privacy.
\end{proof}


\subsection{Communication analysis}
Let us suppose that there are $m$ logistic centers.
Each logistic center, $i$, has a capacity, $lc_i$,
which is the maximum number of agents that can share the logistic center.
Accordingly, we derive the following theorem:

\begin{thm}\label{thm:communication}
In Step 1, the upper bound of the number of communication messages used to find a goal agent is $\sum_{1\leq i\leq m}lc_i$.
\end{thm}
\begin{proof}
In our approach, each agent is only aware of the existence of its own neighbors.
This means that 1) each agent does not know how many neighbors any other agent has, and
2) each agent is not aware of how far away the goal agent is.

As the information regarding logistic centers is public,
all agents know the capacity of each logistic center.
Thus, to guarantee that the query message is able to reach the goal agent,
an agent must assume that 1) each logistic center is using up its capacity, and
2) the goal agent is located in the most distant area.
In this situation, the number of generated communication messages is $\sum_{1\leq i\leq m}lc_i$.
\end{proof}

\textbf{Remark 4}: Theorem \ref{thm:communication} describes the communication overhead in the worst case.
However, as time progresses, this communication overhead can be significantly reduced.
This is because an agent memorizes the plans that it has previously created,
meaning that an agent memorizes the routes to goal agents.
Thus, in the future, an agent can simply exploit a route previously determined to reach a goal agent without the need for communication.
Even if an agent decides to explore a new route,
the communication overhead can be limited by setting the maximum number of query messages during the finding process.
The maximum number of query messages is set to be identical to the number of messages used to find the same goal agent last time.
Formally, we have the following corollary:

\begin{cor}\label{cor:communication}
As time progresses, the communication overhead of each agent monotonically decreases.
\end{cor}
\begin{proof}
Every time an agent explores a new route to a goal agent,
the maximum number of query messages is set to be equal to the number of messages used to find the same goal agent last time.
As each agent memorizes only the shortest routes to goal agents,
only routes that are shorter than these memorized routes will be taken by each agent.
This means that the number of request messages currently being used must be fewer than or equal to the number used previously.
Thus, the communication overhead of each agent monotonically decreases.
\end{proof}

In our approach, the setting of the communication budget $C$ can be controlled by the privacy budget $\epsilon$.
In a multi-agent system, each agent $k$ sets $\epsilon/C$ as their privacy budget
and ceases to communicate when $\epsilon$ is used up.
When $C>\sum_{1\leq i\leq m}lc_i$, the system can guarantee that all communication steps will be completed.
However, a large amount of noise will be added to the system under these circumstances.
When $C<\sum_{1\leq i\leq m}lc_i$, the system is likely to stop before finishing the communication steps.
However, the noise added to	 the system will be limited.
When $C=\sum_{1\leq i\leq m}lc_i$, the system will stop when all communication steps have been completed.
Therefore, by adjusting the privacy budget $\epsilon$ and the communication budget $C$,
the communication overhead of a multi-agent system can be controlled.

\section{Application of our approach to other domains}\label{sec:application}
This section illustrates how our approach can be applied to three other domains:
networks, air travel, and rovers.

\subsection{Packet routing in networks}
In a network, nodes often transmit packets between each other.
These nodes may belong to different areas,
which are connected by routers or access points.
In this domain, a router or access point can be thought of as similar to an agent,
which manages a corresponding area.
In a given area, the information possessed by each node, e.g., its load and performance, is private to the agent.
Moreover, the number of nodes in an area and their communication links are also private to the agent.
Thus, the agents expect that their privacy will be preserved.

As each node has only a limited range of communication, when a node transmits a packet to another node,
the packet may be relayed multiple times by intermediate nodes before reaching its destination.
Since it is highly desirable that nodes receive packets in a timely manner,
the transmission must be efficient so that huge delays can be avoided.
The proposed approach can be applied to create efficient plans for packet routing.

\subsection{Airplane transport}
The airplane transport problem consists of a set of planes and airports.
Moreover, the travel map is partitioned into a set of areas.
In the real world, each area can be thought of as a country.
Therefore, the planes and airports located in a given area are private to the area air traffic controller.
Clearly, each area controller wants to preserve information regarding the status and number of planes and airports in their area as private information.

The airports located on the boundary of two areas are public.
The goal is to transport passengers between airports.
In this problem, each area controller can be thought of as an agent.
When a plane travels from one airport to another,
as the plane has only limited fuel,
passengers may be transferred multiple times on their way to their destination.
Moreover, both area controllers and passengers would clearly prefer the plane to reach its destination as quickly as possible.
Thus, an efficient privacy-preserving planning approach is required.
The proposed approach can be applied to create efficient plans for passenger transport.

\subsection{Rover exploration}
This domain models Mars exploration rovers.
Each rover can be thought of as an agent.
The goal of these rovers is to collect samples.
Each rover has its own private sets of targets and reachable locations.
These targets and reachable locations can be thought of as private facts in our planning model,
the privacy of which must be preserved.

Each rover collects samples in its reachable locations.
When a rover needs to transmit the samples it has collected to another rover,
these samples may have to be transmitted by intermediate rovers in the interim, 
as the number of locations reachable by each rover is limited.
Since samples may decay as time progresses,
it is desirable for the rovers to transmit the samples to the destination as quickly as possible.
Hence, an efficient privacy-preserving planning approach is required.
The proposed approach can be applied to create efficient plans for sample transmission.

In summary, our approach can be applied to all of the planning problems,
in which each party has private information and
local plans can be created by each party using reinforcement learning techniques.
Moreover, reinforcement learning has a broad range of applications, 
including task scheduling in cloud computing \cite{Peng15}, traffic light control \cite{Arel10}, and robot coordination \cite{Kober13}.
Since most of these applications may also have privacy requirements,
our method has the potential to be applied to these real-world scheduling and coordination problems as well.

\tianqing{can we summarize the scenarios that can use our method? }
\dayong{I add a paragraph to summarize this.}

\section{Experiments}\label{sec:experiment}
\subsection{Experimental setup}
The experiments in the present research are conducted based on two scenarios: logistics and packet routing,
which are typical logistic-like problems.
In the logistics scenario, as described in Section \ref{sec:motivation example},
each military base has a set of packages to transport to other military bases.
These military bases may be located in different areas and managed by different military units.
The information pertaining to each military base is private to the managing military unit.

The packet routing scenario is similar to the logistics scenario,
in that each node in an ad hoc network houses a set of packets to be sent to other nodes.
Nodes may belong to different groups and are served by different access points.
The information of each node is private to the serving access point.
The key difference between these two scenarios is that in the packet routing scenario,
new nodes may dynamically join the network
and existing nodes may leave the network at any time, 
while this is not the case for the logistics scenario.
These experiments have also been conducted on the air travel and rover scenarios.
As the results present a similar trend to logistics,
they are not discussed here.

Three evaluation metrics are used in the two scenarios:

1) average route length: the average length of the routes from initial states to goal states;

2) average communication overhead: the average number of communication messages used to make a plan;

3) success rate: the ratio of the number of the successfully transmitted packages/packets to the total number of packages/packets.

In both scenarios, the map shape or network topology is similar to that in Fig. \ref{fig:example}.
The size of the maps/neworks varies from $10$ logistic centers/access points to $50$ logistic centers/access points; 
correspondingly the number of military bases/network nodes varies from $50$ to $250$ 
\footnote{The topologies of maps/networks are created by simulation, 
as most real-world graph datasets \cite{CA,PA,TX} do not contain distance information and thus cannot be used in our experiments. 
We leave the experiments with real-world datasets as one of our future studies.}.
The probability of a package/packet being generated on each military base/node is set to $0.2$.
The communication budget of each agent varies from $C=40$ to $C=80$ depending on variations in the map/network size.
The privacy budget of each agent is set to $\epsilon=0.5$.
Moreover, in the packet routing scenario, during the route finding process,
there is a probability of $0.1$ that an existing node will leave the network and a probability of $0.1$ that a new node will join the network.
The parameter values in the proposed algorithms are chosen experimentally,
and set to $\alpha=0.1$, $\gamma=0.9$ and $\zeta=0.95$.

The proposed planning approach, denoted as \emph{DP-based}, is evaluated in comparison with three closely related approaches. 
The first approach, denoted as \emph{No-privacy}, is also developed by us.
The major features of \emph{No-privacy} are the same as \emph{DP-based}, 
but the privacy-preserving mechanism has been removed. 
Although \emph{No-privacy} is not applicable to privacy-preserving planning, it can be used to evaluate 
how the privacy-preserving mechanism impacts the performance of our \emph{DP-based} approach.
The second approach is based on best-first forward search, denoted as \emph{Best-first},
and has been used in \cite{Nissim14,Brafman15,Stolba17}.
In the \emph{Best-first} approach, when an agent transmits a package/packet to a logistic center/access point,
the agent broadcasts this state to all the other agents.
The nearest agent takes the package/packet based on this state
and transmits it to the next logistic center/access point.
This process continues until the goal agent is reached.
The third approach is GPPP (greedy privacy-preserving planner), denoted as \emph{Greedy},
which was developed in \cite{Maliah17}.
The \emph{Greedy} approach consists of two phases: global planning and local planning.
In the global planning phase, all agents collaboratively devise a global plan using a best-first search method.
Next, in the local planning phase, each agent creates a local plan by executing a single-agent planning procedure.

\subsection{Experimental results}
\subsubsection{The logistics scenario}
\begin{figure}[ht]
\centering
	\begin{minipage}{0.48\textwidth}
   \subfigure[\scriptsize{Average route length}]{
    \includegraphics[width=0.32\textwidth, height=2.5cm]{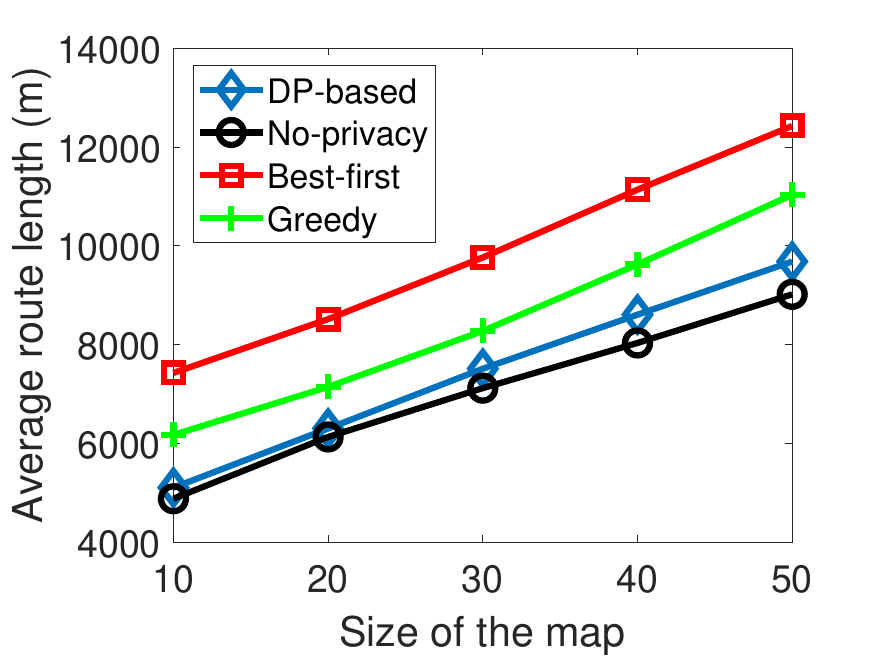}
			\label{fig:S1Route}}
    \subfigure[\scriptsize{Average communication overhead}]{
    \includegraphics[width=0.32\textwidth, height=2.5cm]{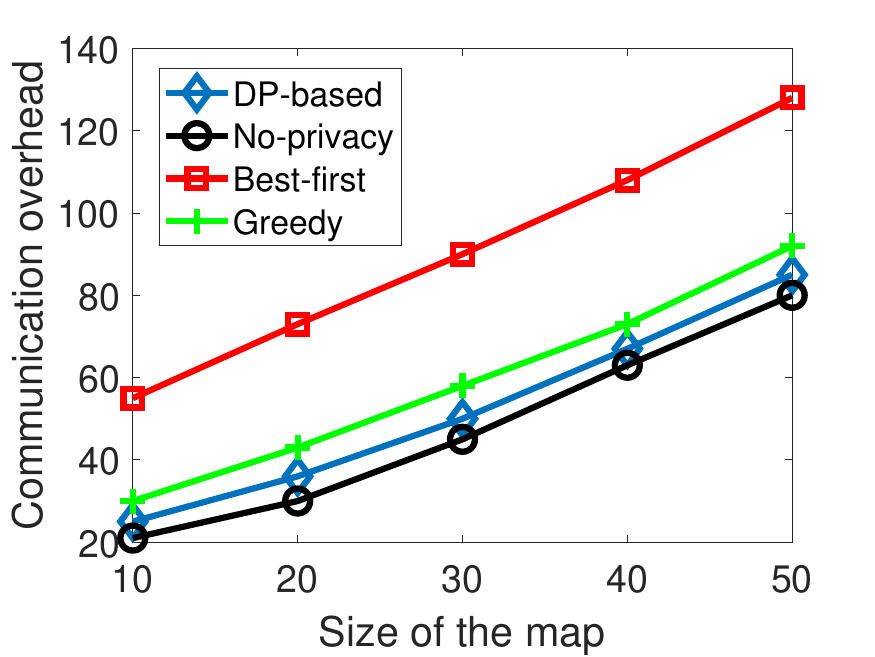}
			\label{fig:S1Communication}}
    \subfigure[\scriptsize{Success rate}]{
    \includegraphics[width=0.32\textwidth, height=2.5cm]{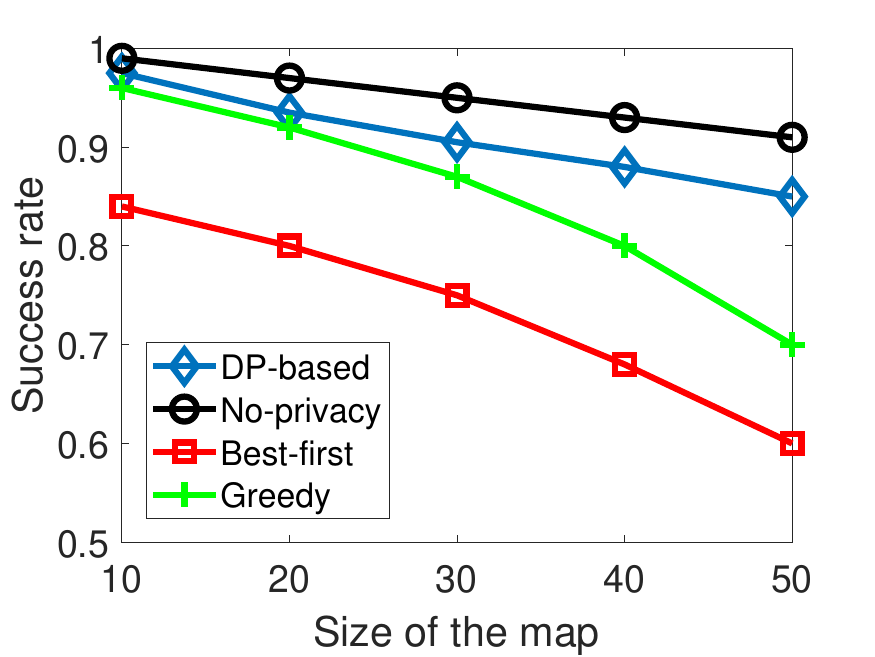}
			\label{fig:S1Success}}\\
    \end{minipage}
	\caption{Performance of the four approaches on the logistics scenario with variation of the map size}
	\label{fig:S1}
\end{figure}

Fig. \ref{fig:S1} demonstrates the performance of the four approaches on the logistics scenario with variation of the map size.
As the map size grows larger, for all four approaches,
the average route length and the average communication overhead progressively increase,
while the success rate gradually decreases.

As the map size increases, the distance between an original agent and a destination agent may be enlarged accordingly.
Therefore, the average route length increases.
Moreover, when this occurs,
the number of intermediate agents also increases.
Thus, the average communication overhead rises as well.
Due to this increase in the average communication overhead,
the communication budget of some agents may be used up before a plan is made.
Hence, the success rate reduces.

The proposed \emph{DP-based} approach achieves much better performance than the \emph{Best-first} and \emph{Greedy} approaches.
The reinforcement learning algorithm in the \emph{DP-based} approach can find shorter routes than the other two approaches.
Moreover, in the \emph{DP-based} approach, agents are allowed to communicate only with neighbors,
and a privacy budget is adopted to control communication overhead.
Thus, the \emph{DP-based} approach uses less communication overhead than the other two approaches.
In addition, the \emph{DP-based} approach successfully makes more plans than the other two approaches
before the communication budget is used up. 
Overall, the performance of \emph{No-privacy} approach is slightly better than the \emph{DP-based} approach. 
As privacy is not taken into account in the \emph{No-privacy} approach, 
the information shared between agents is accurate, 
and agents can make accurate plans based on this accurate information. 
However, the private information of each agent is entirely disclosed to other agents under this approach, 
a situation that should be avoided in real-world applications. 
More specifically, in Fig. \ref{fig:S1Route}, the average route length in the \emph{DP-based} approach is only about $2\%$ longer than for the \emph{No-privacy} approach. 
This is because in the \emph{DP-based} approach, a plan is made up of 
a set of local plans created by the initial agent and the intermediate agents. 
Each of these local plans is created by an individual agent with reference to its private but accurate information. 
Since most of the information used to create a plan is accurate, 
the introduction of our privacy-preserving mechanism does not substantially impact the average route length. 

The \emph{Best-first} approach achieves the worst performance out of the four approaches.
In the \emph{Best-first} approach, a package is transmitted to the nearest agent.
However, in large and complex maps, the nearest agent may not always be the best choice.
Moreover, always choosing the nearest agent may result in a transmission loop;
if this situation arises, packages will never reach their destinations.
In comparison, the performance of \emph{Greedy} approach is better than the \emph{Best-first} approach,
as the \emph{Greedy} approach features a global planning phase that involves selecting the appropriate logistic centers to create a high-level route,
which conserves communication overhead.

\begin{figure}[ht]
\centering
	\begin{minipage}{0.48\textwidth}
   \subfigure[\scriptsize{Average route length}]{
    \includegraphics[width=0.32\textwidth, height=2.5cm]{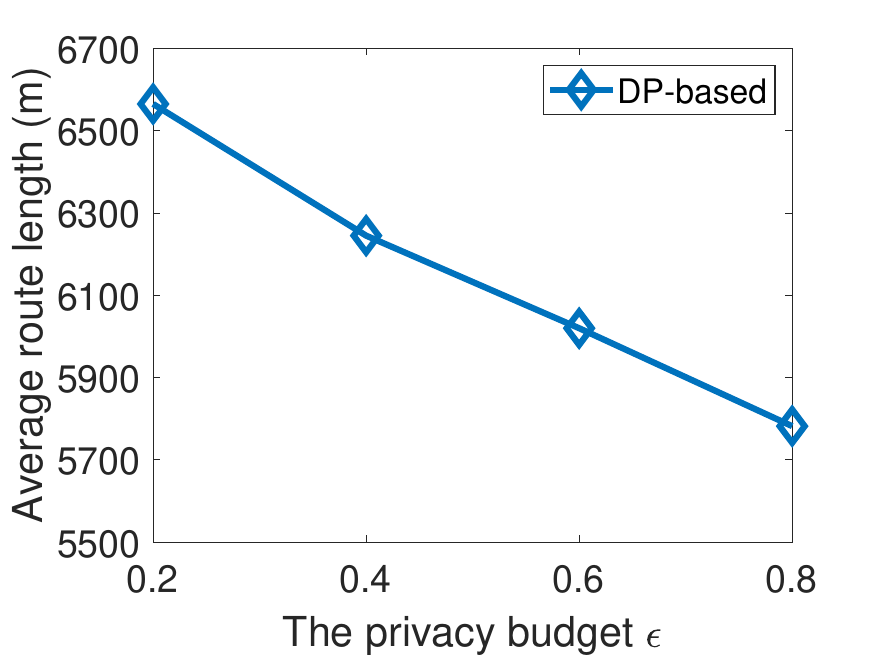}
			\label{fig:S1PriRoute}}
    \subfigure[\scriptsize{Average communication overhead}]{
    \includegraphics[width=0.32\textwidth, height=2.5cm]{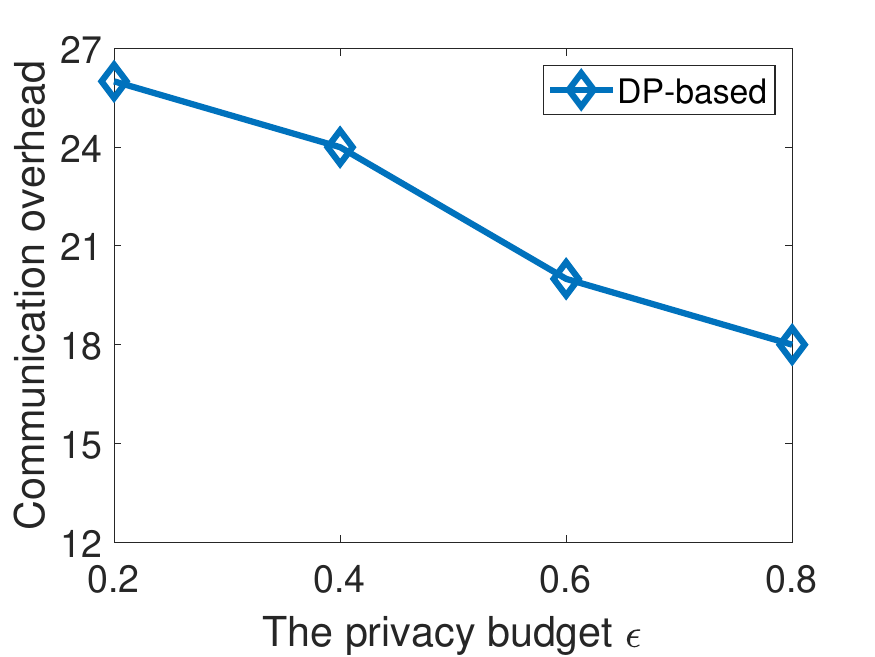}
			\label{fig:S1PriComm}}
    \subfigure[\scriptsize{Success rate}]{
    \includegraphics[width=0.32\textwidth, height=2.5cm]{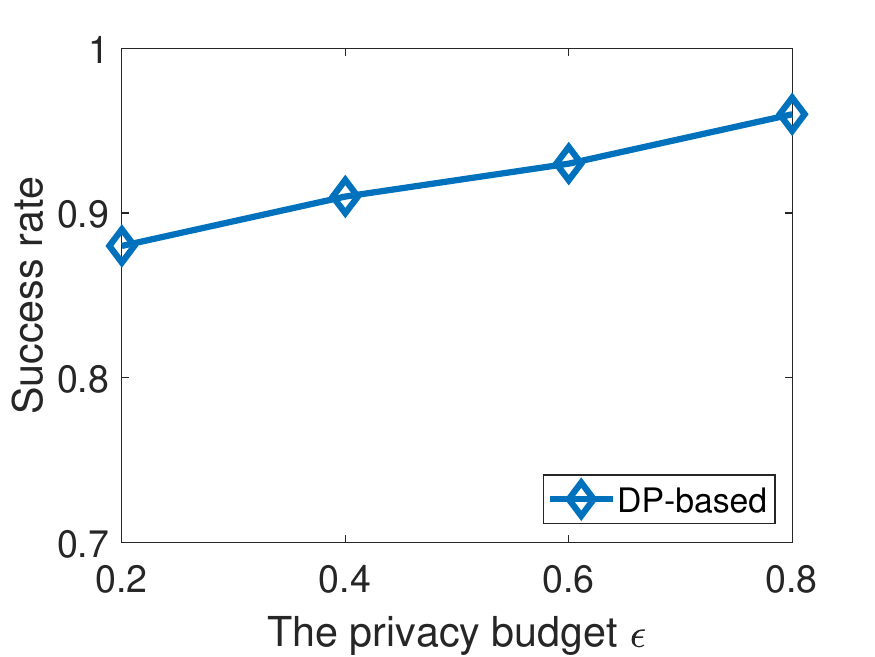}
			\label{fig:S1PriSuccess}}\\
    \end{minipage}
	\caption{Performance of the \emph{DP-based} approach on the logistics scenario with variation of the privacy budget value}
	\label{fig:S1Pri}
\end{figure}

Fig. \ref{fig:S1Pri} demonstrates the performance of the \emph{DP-based} approach on the logistics scenario 
with variation of the privacy budget $\epsilon$ value from $0.2$ to $0.8$. 
The number of logistic centers is fixed at $10$.
It can be seen that with the increase of the privacy budget $\epsilon$ value,
the performance of the \emph{DP-based} approach improves,
namely it achieves a shorter average route length (Fig. \ref{fig:S1PriRoute}),
lower average communication overhead (Fig. \ref{fig:S1PriComm}),
and higher success rate (Fig. \ref{fig:S1PriSuccess}).
According to the Laplace mechanism, when the $\epsilon$ value is small,
the noise, added to the map, is large.
A large noise value will significantly affect the agents planning.
For example, agent $a$ has two neighbors $b$ and $c$.
Now, suppose that 1) agent $a$ wants to send a package to $d$,
and 2) delegating the package to $b$ is a better choice than $c$.
However, when agents $b$ and $c$ obfuscate their maps,
due to the large noise, the obfuscation results may make $c$ appear to be a better choice than $b$.
Thus, agent $a$ may make a sub-optimal plan.
This situation is alleviated
when the $\epsilon$ value increases.

\subsubsection{The packet routing scenario}
\begin{figure}[ht]
\centering
	\begin{minipage}{0.48\textwidth}
   \subfigure[\scriptsize{Average route length}]{
    \includegraphics[width=0.32\textwidth, height=2.5cm]{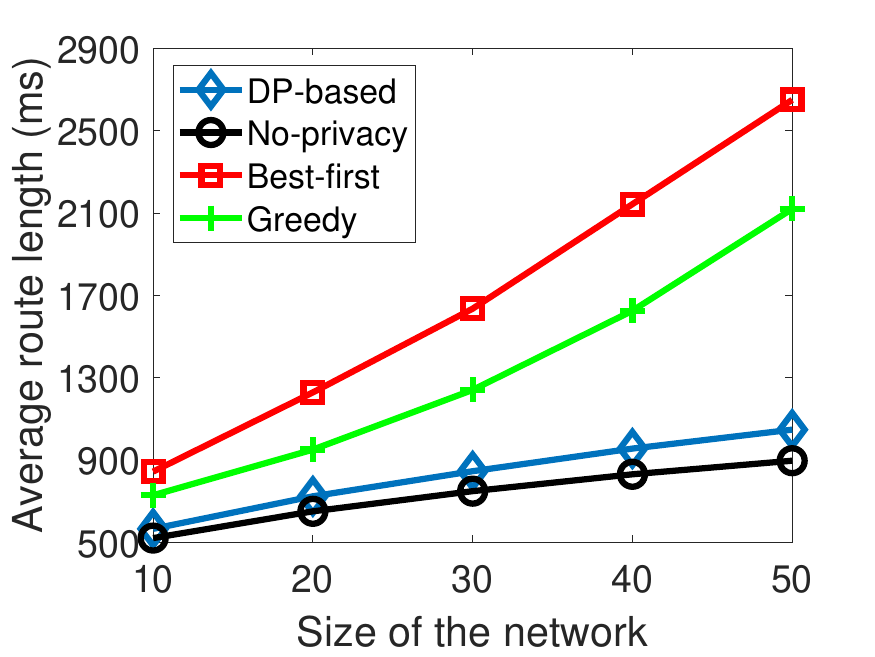}
			\label{fig:S2Route}}
    \subfigure[\scriptsize{Average communication overhead}]{
    \includegraphics[width=0.32\textwidth, height=2.5cm]{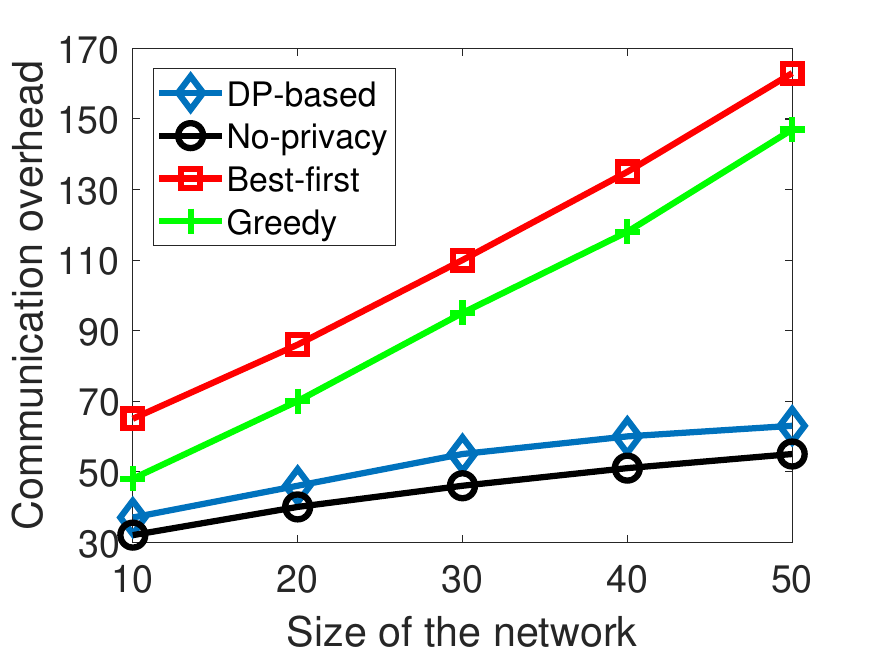}
			\label{fig:S2Communication}}
    \subfigure[\scriptsize{Success rate}]{
    \includegraphics[width=0.32\textwidth, height=2.5cm]{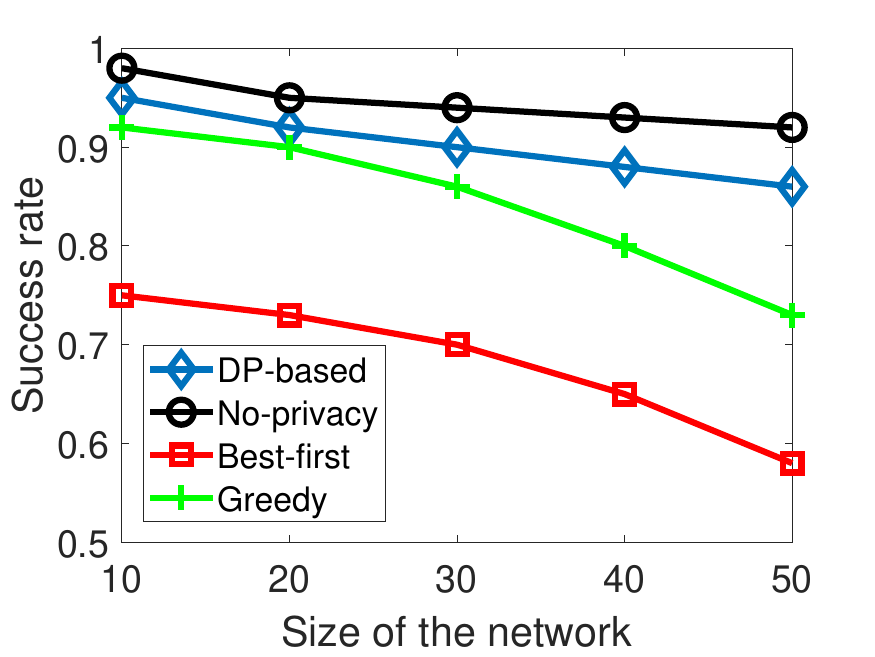}
			\label{fig:S2Success}}\\
    \end{minipage}
	\caption{Performance of the four approaches on the packet routing scenario with variation of the network size}
	\label{fig:S2}
\end{figure}

\begin{figure}[ht]
\centering
	\begin{minipage}{0.48\textwidth}
   \subfigure[\scriptsize{Average route length}]{
    \includegraphics[width=0.32\textwidth, height=2.5cm]{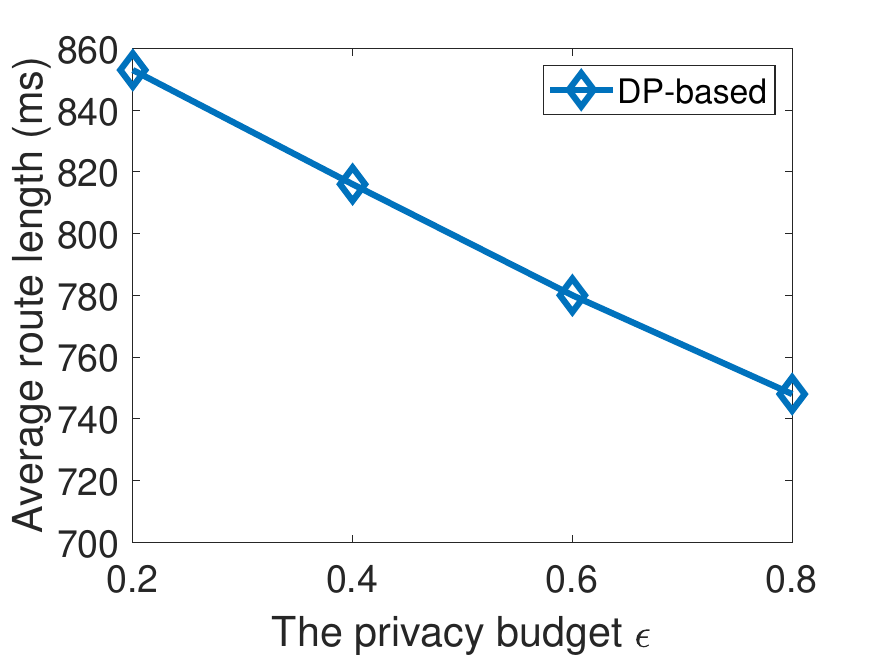}
			\label{fig:S2PriRoute}}
    \subfigure[\scriptsize{Average communication overhead}]{
    \includegraphics[width=0.32\textwidth, height=2.5cm]{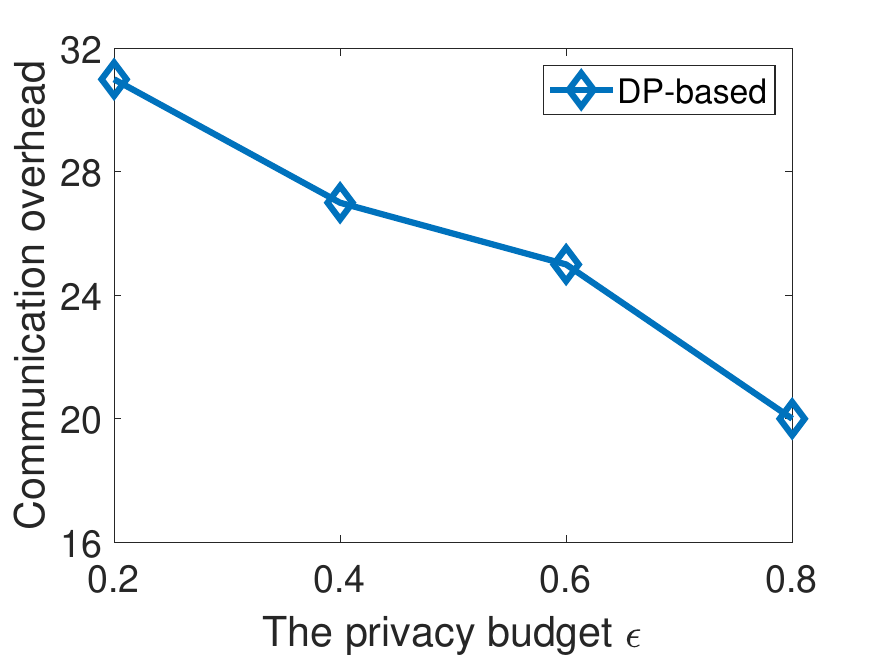}
			\label{fig:S2PriComm}}
    \subfigure[\scriptsize{Success rate}]{
    \includegraphics[width=0.32\textwidth, height=2.5cm]{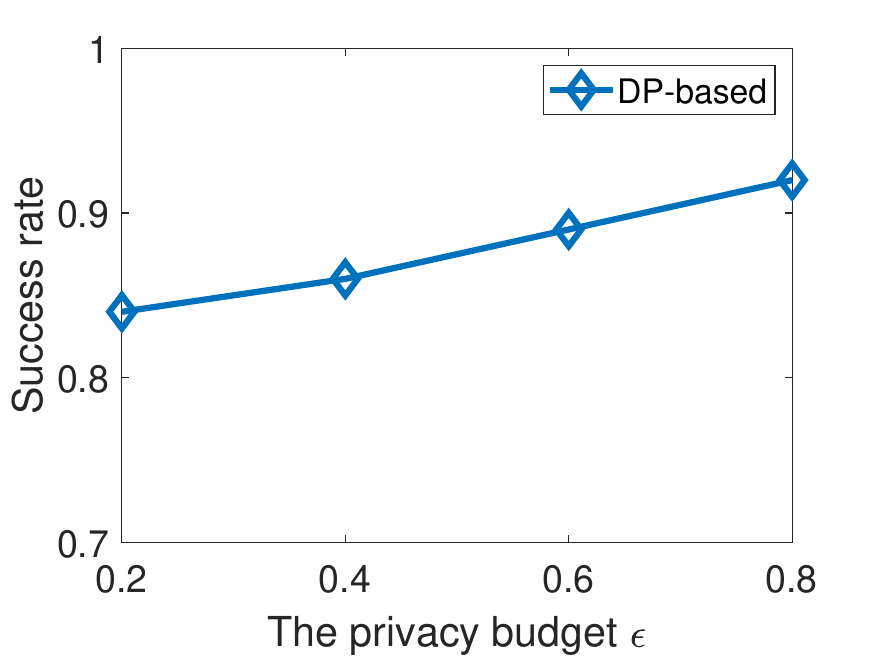}
			\label{fig:S2PriSuccess}}\\
    \end{minipage}
	\caption{Performance of the \emph{DP-based} approach on the packet routing scenario with variation of the privacy budget}
	\label{fig:S2Pri}
\end{figure}

Fig. \ref{fig:S2} illustrates the performance of the four approaches on the packet routing scenario with variation of the network size,
while Fig. \ref{fig:S2Pri} depicts the performance of the \emph{DP-based} approach on the packet routing scenario with variation of the privacy budget $\epsilon$.
After comparing Fig. \ref{fig:S1} to Fig. \ref{fig:S2} and Fig. \ref{fig:S1Pri} to Fig. \ref{fig:S2Pri},
it can be concluded that these approaches exhibit similar trends in terms of their results on the two scenarios,
but that the performance of these approaches is worse on the packet routing scenario than on the logistic scenario.
This is mainly due to the dynamism of the packet routing scenario.
When a node leaves the network, the routes involving that node are broken.
Thus, agents have to re-find routes.
This incurs extra communication overhead and reduces success rates to some extent.

\begin{figure}[ht]
\centering
	\begin{minipage}{0.48\textwidth}
   \subfigure[\scriptsize{Average route length}]{
    \includegraphics[width=0.32\textwidth, height=2.5cm]{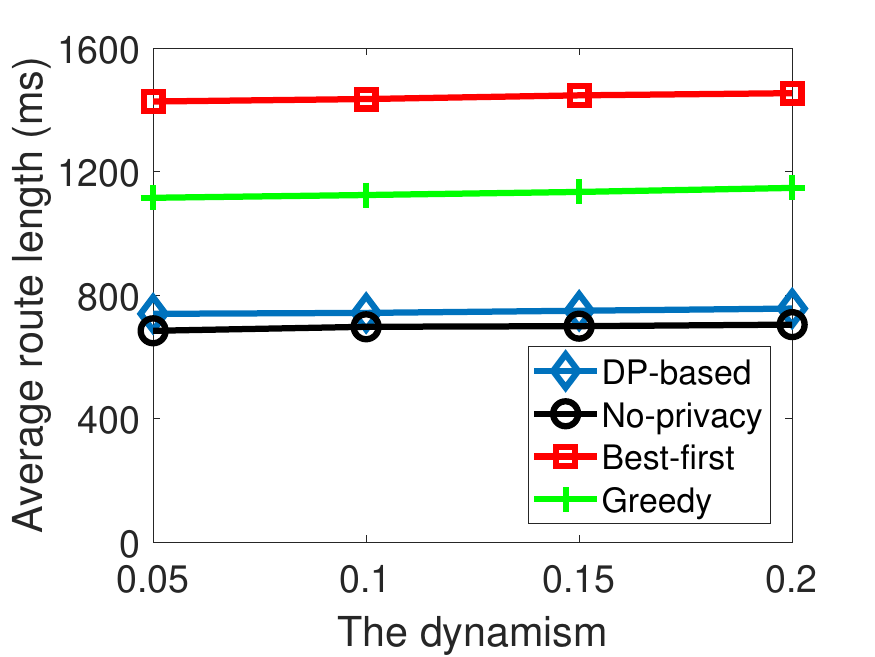}
			\label{fig:S2DyRoute}}
    \subfigure[\scriptsize{Average communication overhead}]{
    \includegraphics[width=0.32\textwidth, height=2.5cm]{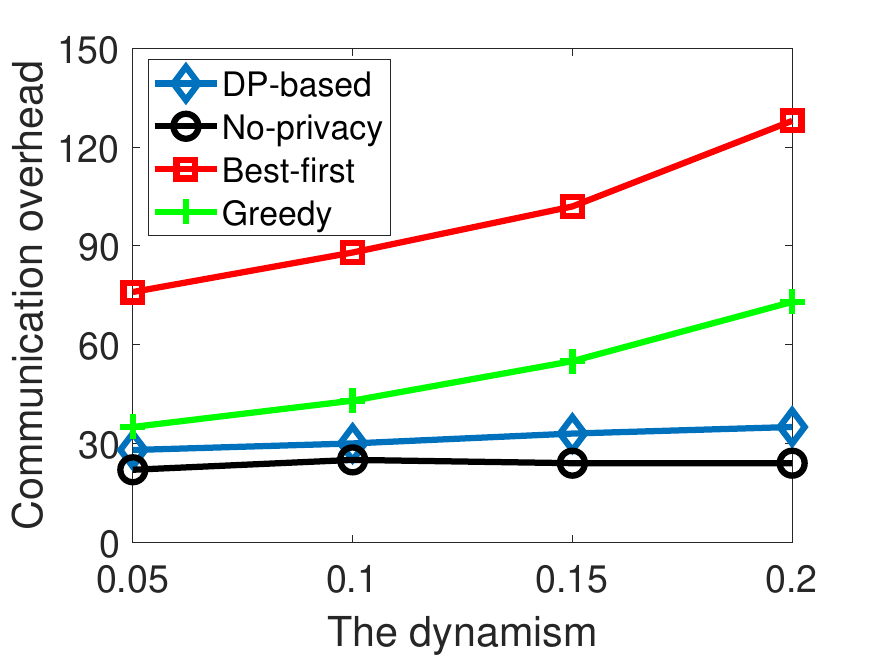}
			\label{fig:S2DyCommunication}}
    \subfigure[\scriptsize{Success rate}]{
    \includegraphics[width=0.32\textwidth, height=2.5cm]{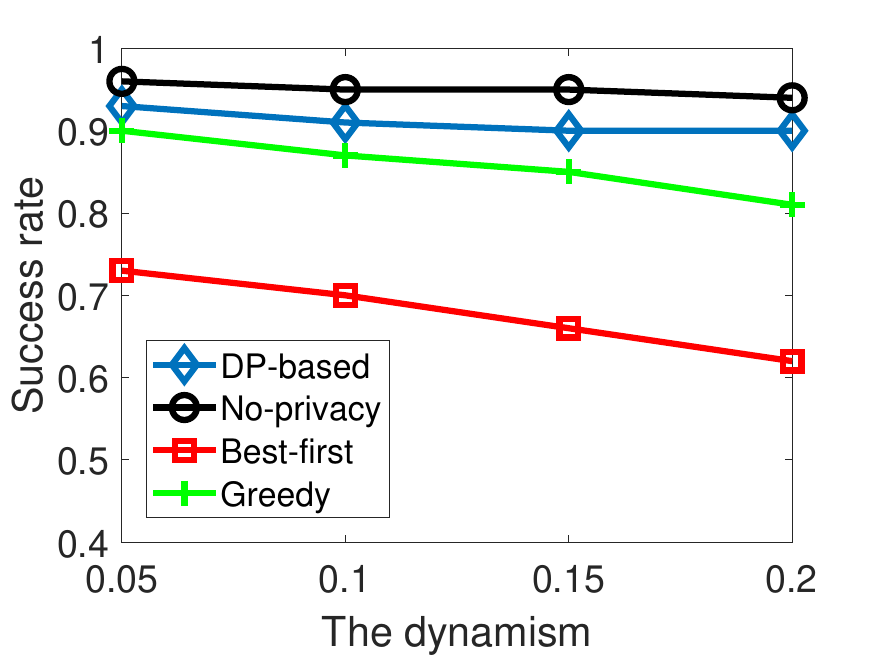}
			\label{fig:S2DySuccess}}\\
    \end{minipage}
	\caption{Performance of the three approaches on the packet routing scenario with variation of the dynamism}
	\label{fig:S2Dy}
\end{figure}

Fig. \ref{fig:S2Dy} illustrates the performance of the four approaches on the packet routing scenario with variation of the dynamism,
such that the probability of a node leaving or joining the network varies from $0.05$ to $0.2$ and the network size is fixed at $10$ access points.
From Fig. \ref{fig:S2Dy}, it can be seen that an increase in the dynamism negatively affects the \emph{Best-first} and \emph{Greedy} approaches in terms of their average communication overhead and success rates,
but does not significantly impact the \emph{DP-based} and \emph{No-privacy} approaches.

As the dynamism increases, the frequency with which nodes leave or join the network also increases.
Thus, the number of affected routes increases as well.
In the \emph{Best-first} and \emph{Greedy} approaches, when a route is broken,
a new finding process is launched.
This may not significantly affect the average route length (Fig. \ref{fig:S2DyRoute}),
as route length depends on the positions of nodes rather than the number of nodes.
However, launching a new finding process results in additional communication overhead,
and may thus reduce success rates due to depletion of the communication budget.
By contrast, the \emph{DP-based} and \emph{No-privacy} approaches do not require a new finding process
when a route is broken.
In the \emph{DP-based} approach, routes are found by using reinforcement learning on obfuscated local network topologies.
These obfuscated local network topologies are obtained using differential privacy.
Differential privacy can guarantee that
a node being brought in or out of a local network will have minimum effect on the statistical information.
Therefore, when a node leaves or joins a local network,
the serving access point does not need to re-obfuscate the new network
or to communicate with the original access point about the change in the network.
Hence, the communication budget can be conserved,
and the success rate is preserved.

\subsection{Summary}
According to the experimental results, the proposed \emph{DP-based} approach achieves better results than the \emph{Best-first} and \emph{Greedy} approaches in all experimental situations considered here.
The average length of routes found by the \emph{DP-based} approach is about $25\%$ and $15\%$ shorter, respectively, than those found using the \emph{Best-first} and \emph{Greedy} approaches.
The \emph{DP-based} approach also uses about $20\%$ and $10\%$ less communication overhead than the \emph{Best-first} and \emph{Greedy} approaches, respectively.
Moreover, the \emph{DP-based} approach achieves about $10\%$ and $5\%$ higher success rates than the \emph{Best-first} and \emph{Greedy} approaches, respectively. 

Regarding performance, the \emph{DP-based} approach is slightly worse than the \emph{No-privacy} approach 
by a factor of about $3\%$ in terms of average route length, $2\%$ in communication overhead and $2\%$ in success rate. 
The \emph{DP-based} approach, however, strongly protects the privacy of agents, 
which is entirely disregarded in the \emph{No-privacy} approach.
Therefore, based on the experimental results, the efficiency of our \emph{DP-based} approach can be proven.

\section{Conclusion and Future Work}\label{sec:conclusion}
This paper proposes a novel strong privacy-preserving planning approach for logistic-like problems.
In this approach, an agent creates a complete plan by using obfuscated private information from each intermediate agent, 
where this obfuscation is achieved by adopting the differential privacy technique.
Due to the advantages of differential privacy,
following obfuscation, an agent's private information cannot be deduced by other agents regardless of their reasoning power.
\tianqing{please use two sentences to summrize the proposed method}
\dayong{done}
This approach is the first in existence to achieve strong privacy, completeness and efficiency simultaneously by taking advantage of differential privacy.
Moreover, this approach is communication-efficient. 
Compared to the benchmark approaches, our approach achieves better performance in various aspects.
\tianqing{the proposed method can be applied in ******}
\dayong{done}

In the future, we intend to extend our approach by introducing malicious agents.
Existing approaches commonly assume that agents are honest but curious.
Introducing malicious agents, which provide false information to others, 
may be a challenging and interesting addition to the field of multi-agent planning. 
Also, as described in the experimental part (Section \ref{sec:experiment}), 
we will continue to search usable real-world datasets and evaluate our approach with them.


\bibliographystyle{IEEEtran}
{\small \bibliography{references}}

\end{document}